\documentclass[10pt,twocolumn,letterpaper]{article}

\usepackage[pagenumbers]{cvpr} 

\usepackage[dvipsnames]{xcolor}
\usepackage{amsthm}

\definecolor{darkred}{rgb}{0.7,0.1,0.1}
\definecolor{darkgreen}{rgb}{0.1,0.7,0.1}
\definecolor{cyan}{rgb}{0.7,0.0,0.7}
\definecolor{dblue}{rgb}{0.2,0.2,0.8}
\definecolor{maroon}{rgb}{0.76,.13,.28}
\definecolor{burntorange}{rgb}{0.81,.33,0}
\definecolor{tealblue}{rgb}{0.212,0.459, 0.533}
\definecolor{darkyellow}{rgb}{0.6 , 0.5, 0.0}
\definecolor{mybrown}{rgb}{0.87058824, 0.56078431, 0.01960784}
\definecolor{myorange}{rgb}{0.835, 0.368, 0}
\definecolor{mygt}{rgb}{0.0078125 , 0.57421875, 0.40625}
\definecolor{mysp}{rgb}{0.84765625, 0.515625  , 0.0234375}
\definecolor{mypink}{rgb}{0.93359375, 0.62109375, 0.83984375}
\definecolor{mypurple}{rgb}{0.5372549 , 0.20588235, 0.99372549}
\definecolor{myblue}{HTML}{185ABC} 	
\definecolor{mygreen}{HTML}{137333}
\definecolor{carmine}{rgb}{0.59, 0.0, 0.09}
\definecolor{cornflowerblue}{rgb}{0.39, 0.58, 0.93}
\definecolor{brightmaroon}{rgb}{0.76, 0.13, 0.28}
\definecolor{pp}{rgb}{0.43921569, 0.18823529, 0.62745098}
\definecolor{rr}{rgb}{0.5254902 , 0.00784314, 0.12941176}
\definecolor{bb}{rgb}{0.0, 0.12549019607,0.57647058823}
\definecolor{yy}{rgb}{0.49803922, 0.3372549 , 0.0}
\definecolor{gg}{rgb}{0.02352941, 0.3372549 , 0.17647059}
\definecolor{highlightRowColor}{rgb}{0.95, 0.95, 1}

\newcommand{\eat}[1]{} 
\newcommand{\NAME}{StableDreamer\xspace}
\newcommand{\OURS}{\NAME}
\newcommand{\GS}{3D Gaussians\xspace}

\makeatletter
\usepackage{xspace}
\def\@onedot{\ifx\@let@token.\else.\null\fi\xspace}
\DeclareRobustCommand\onedot{\futurelet\@let@token\@onedot}

\def\eg{\emph{e.g}\onedot} 
\def\ie{\emph{i.e}\onedot}

\usepackage{amsmath,amsfonts,bm}

\def\eqref#1{eq.(\ref{#1})}

\def\1{\bm{1}}

\def\vx{{\bm{x}}}

\DeclareMathAlphabet{\mathsfit}{\encodingdefault}{\sfdefault}{m}{sl}
\SetMathAlphabet{\mathsfit}{bold}{\encodingdefault}{\sfdefault}{bx}{n}

\newcommand{\E}{\mathbb{E}}

\definecolor{cvprblue}{rgb}{0.21,0.49,0.74}
\usepackage[pagebackref,breaklinks,colorlinks,citecolor=cvprblue]{hyperref}
\usepackage{adjustbox}
\usepackage{multirow}

\newtheorem{prop}{Proposition}

\def\paperID{10242} 
\def\confName{CVPR}
\def\confYear{2024}

\title{\NAME: Taming Noisy Score Distillation Sampling for Text-to-3D}

\author{
Pengsheng Guo \quad Hans Hao \quad Adam Caccavale \quad Zhongzheng Ren \quad Edward Zhang \\  Qi Shan \quad Aditya Sankar \quad Alexander G. Schwing \quad  Alex Colburn \quad Fangchang Ma \\
Apple
}

\begin{document}

\twocolumn[{
\renewcommand{\twocolumn}[1][]{#1}
\maketitle
\newcommand{\myimage}[2]{
    \begin{minipage}[t]{0.191\textwidth}
        \captionsetup[figure]{labelformat=empty} 
        \captionsetup{font=scriptsize,justification=centering}
        \includegraphics[width=\linewidth]{images/ours/#1}
        \vspace{-21pt}
        \captionof{figure}{#2}
    \end{minipage}
}
\setlength{\tabcolsep}{1pt}
\begin{tabular}{ccccc} 
\centering
\myimage{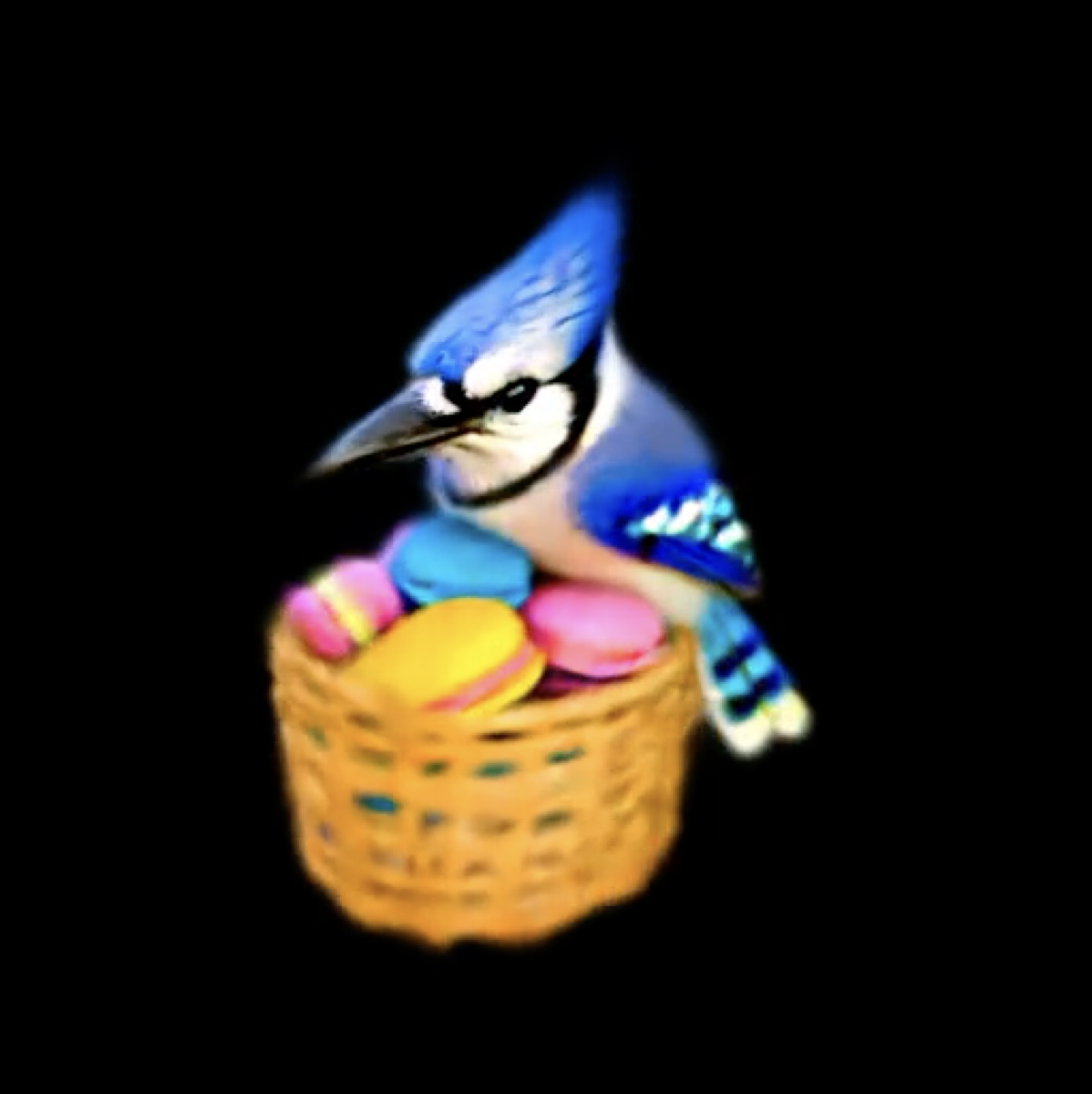}{blue jay standing on large\\ basket of rainbow macarons} 
& \myimage{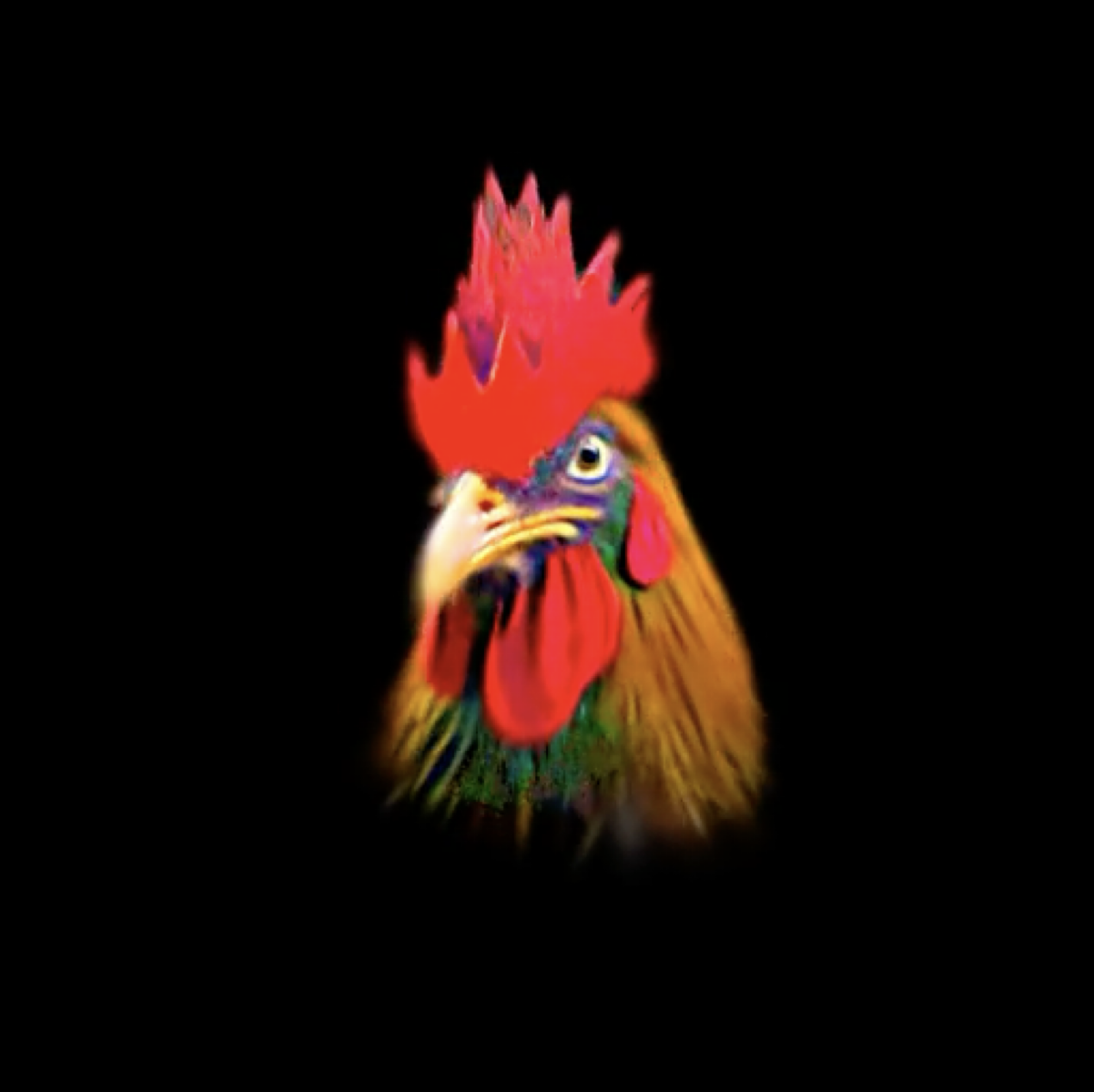}{a colorful rooster} 
& \myimage{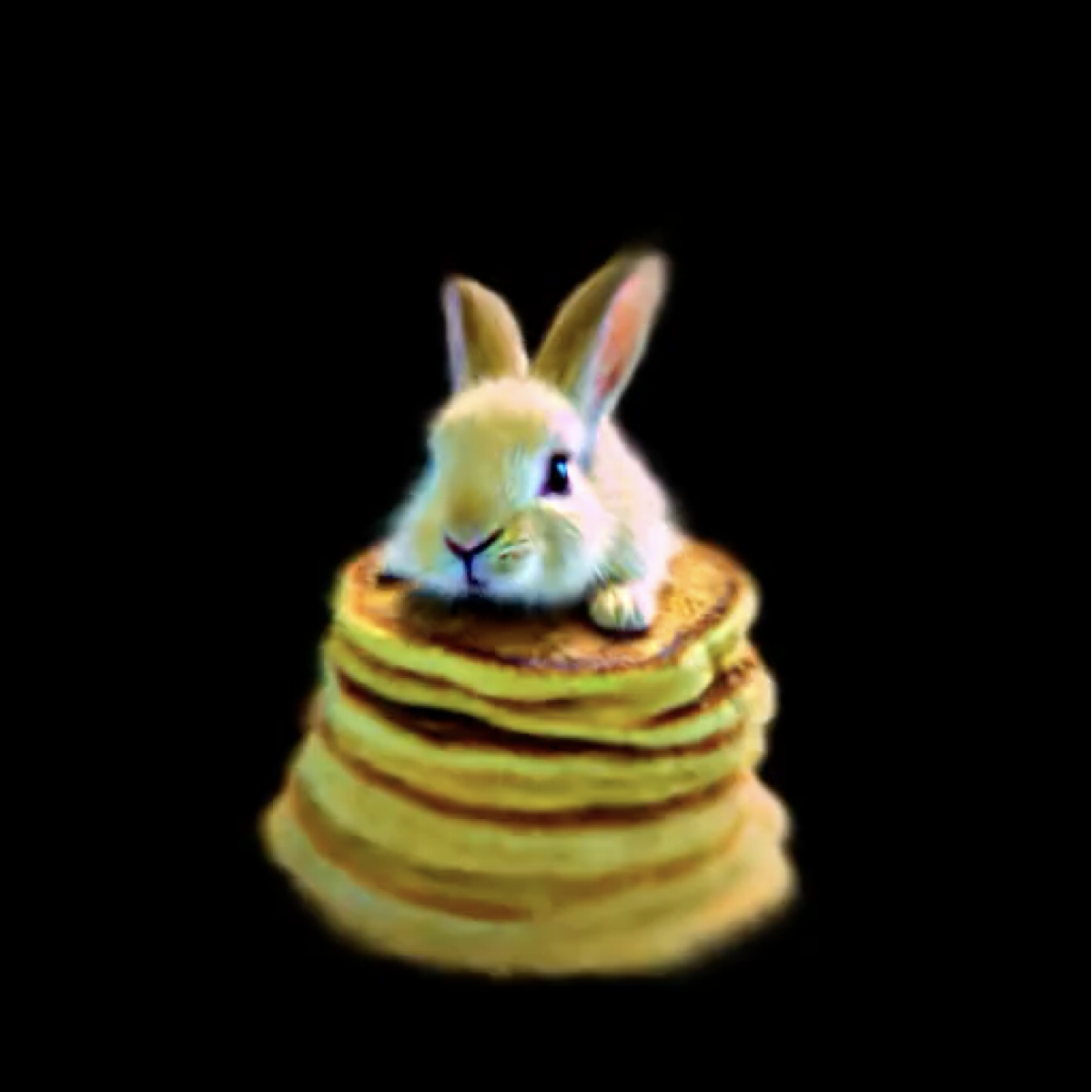}{a baby bunny sitting on top \\of a stack of pancakes} 
& \myimage{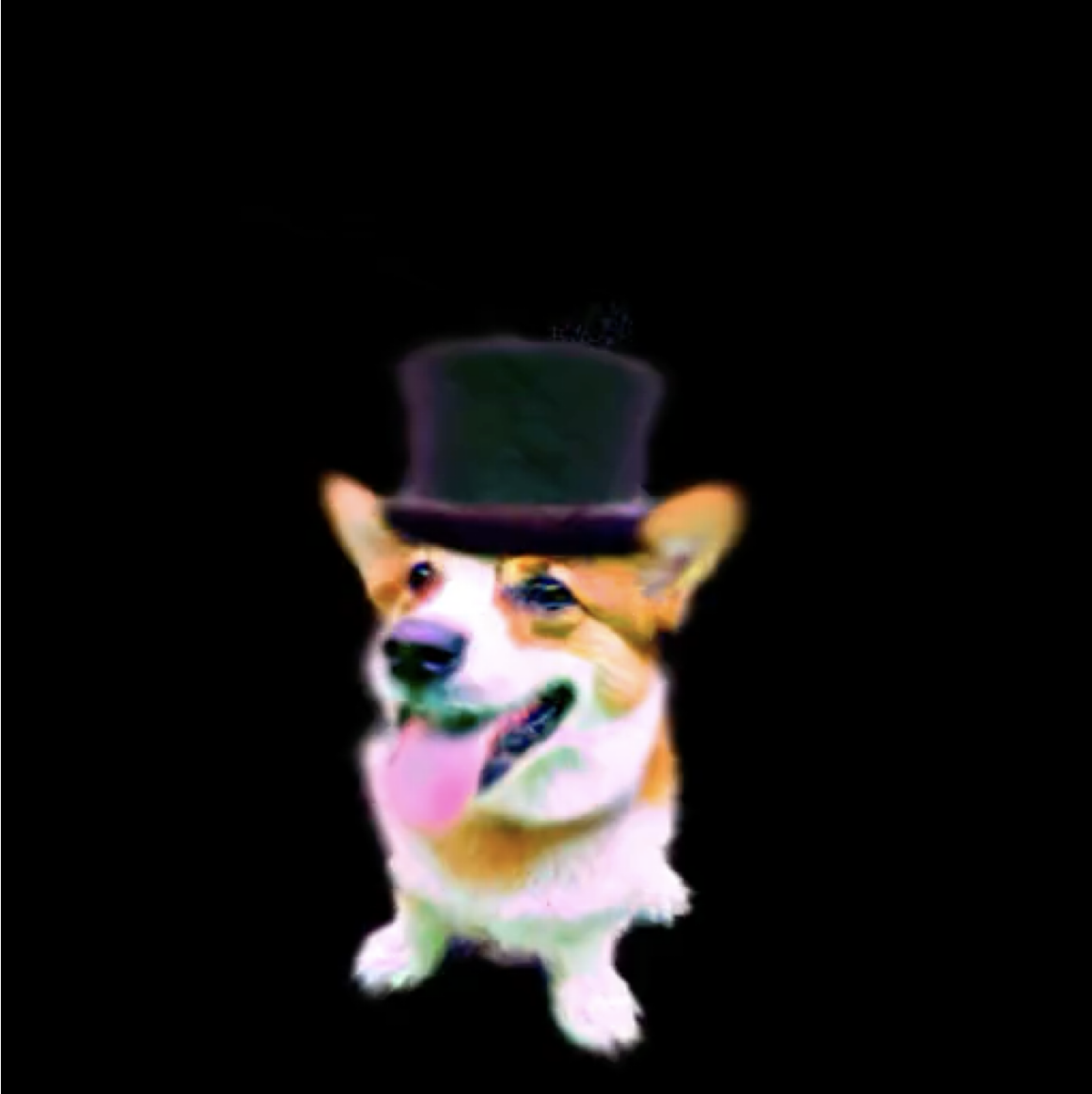}{a corgi wearing a top hat} 
& \myimage{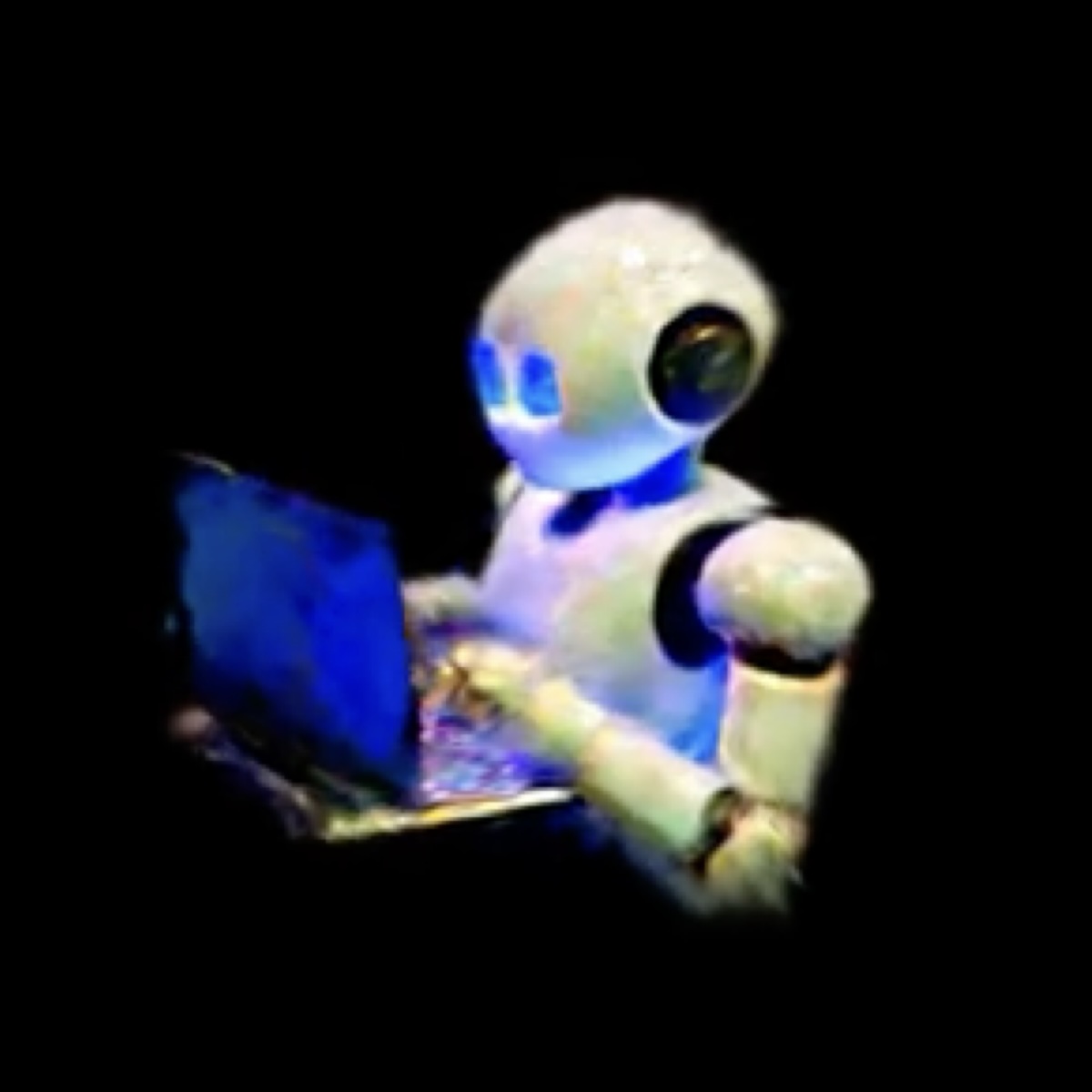}{a DSLR photo of a humanoid robot using a laptop}
\\ 
\myimage{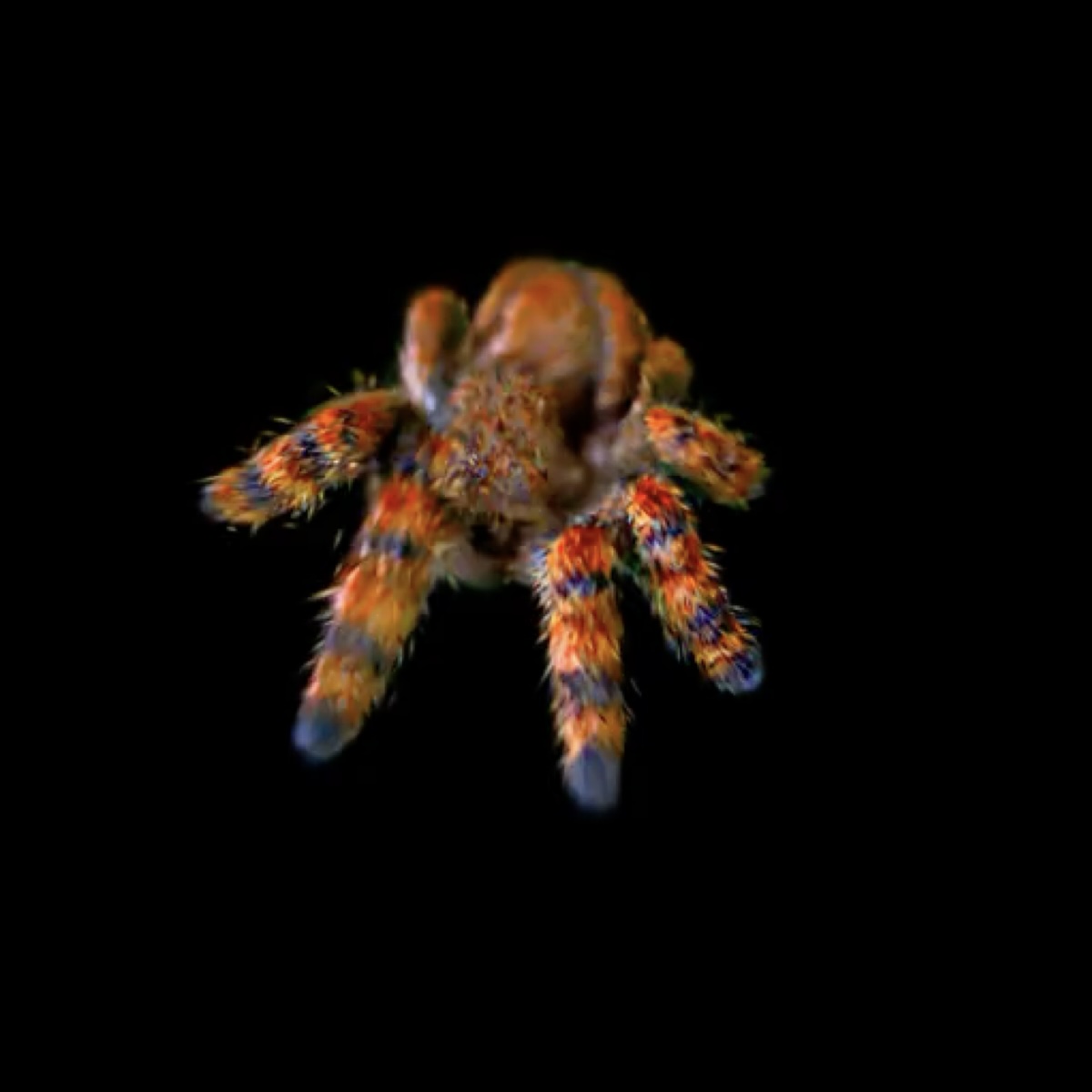}{a tarantula, highly detailed} 
&
\myimage{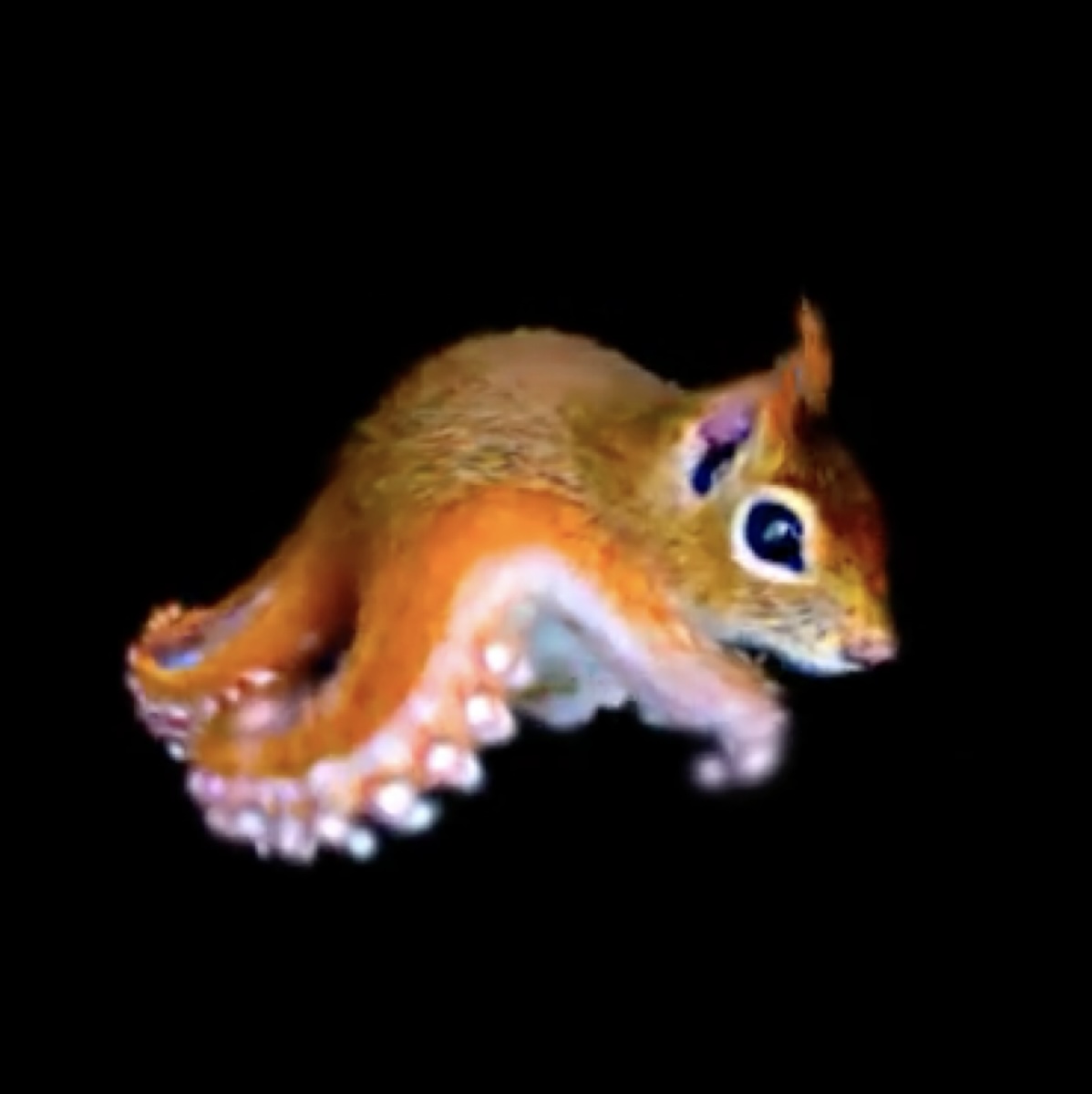}{a squirrel-octopus hybrid}
&
\myimage{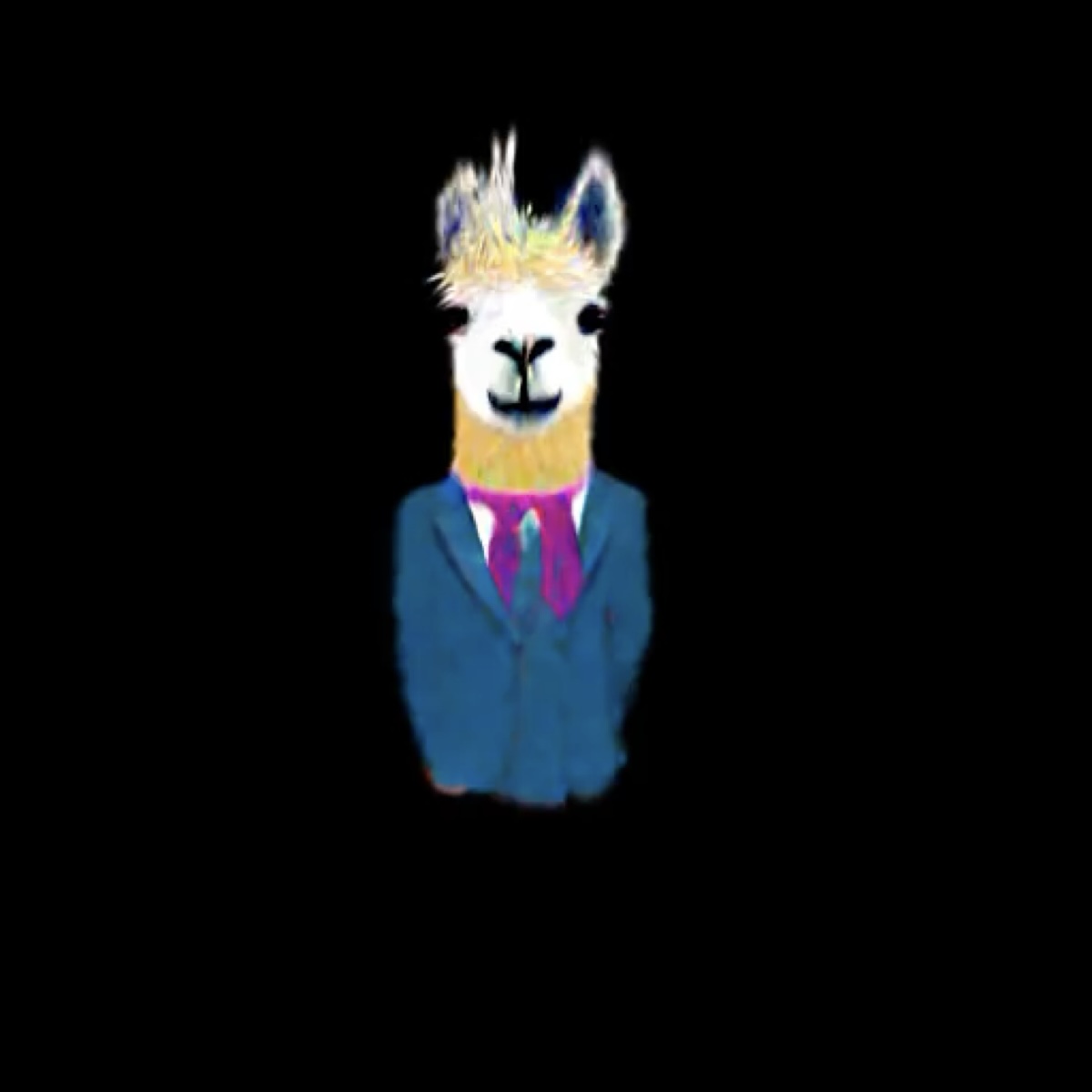}{a llama wearing a suit}
&
\myimage{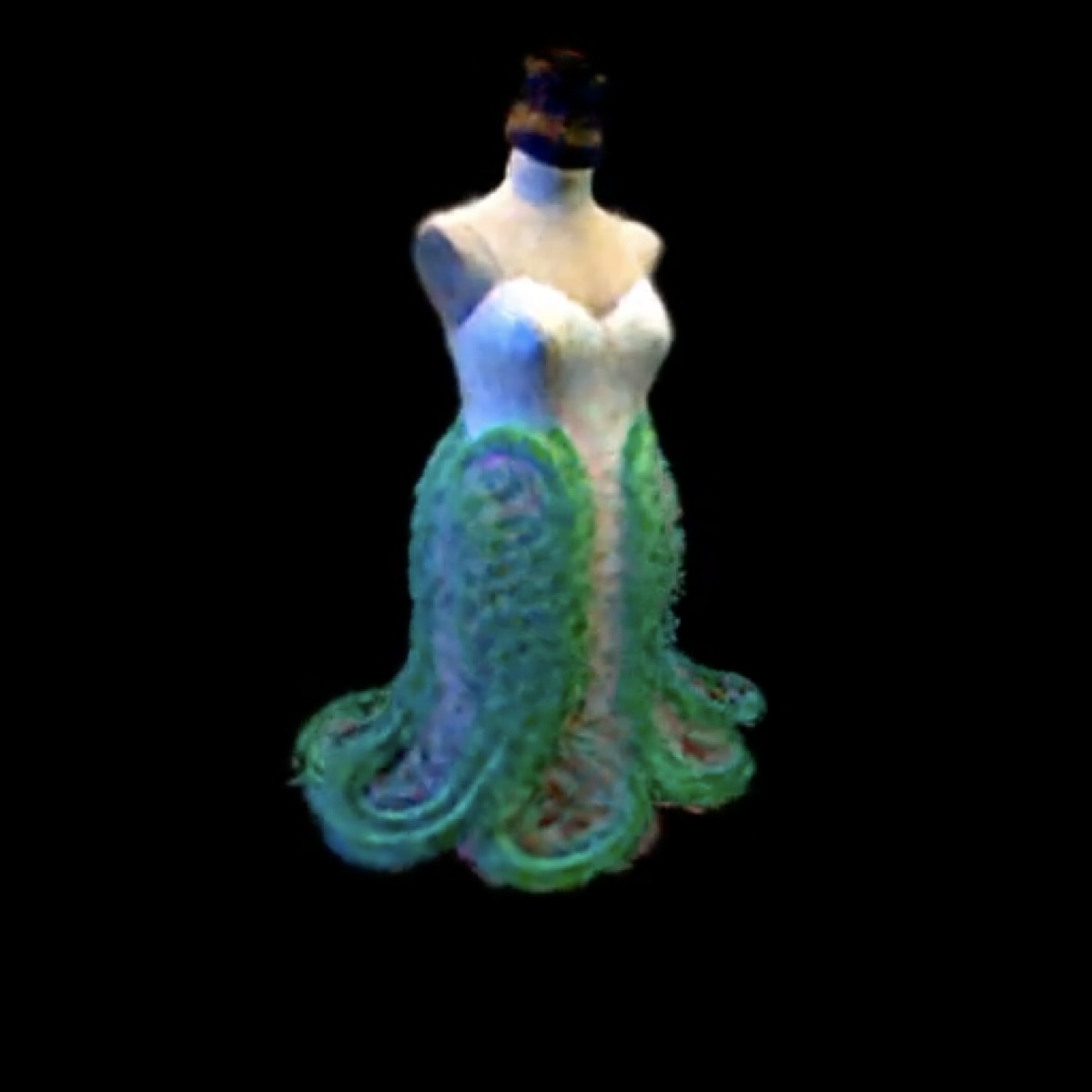}{Wedding dress made of tentacles}
&
\myimage{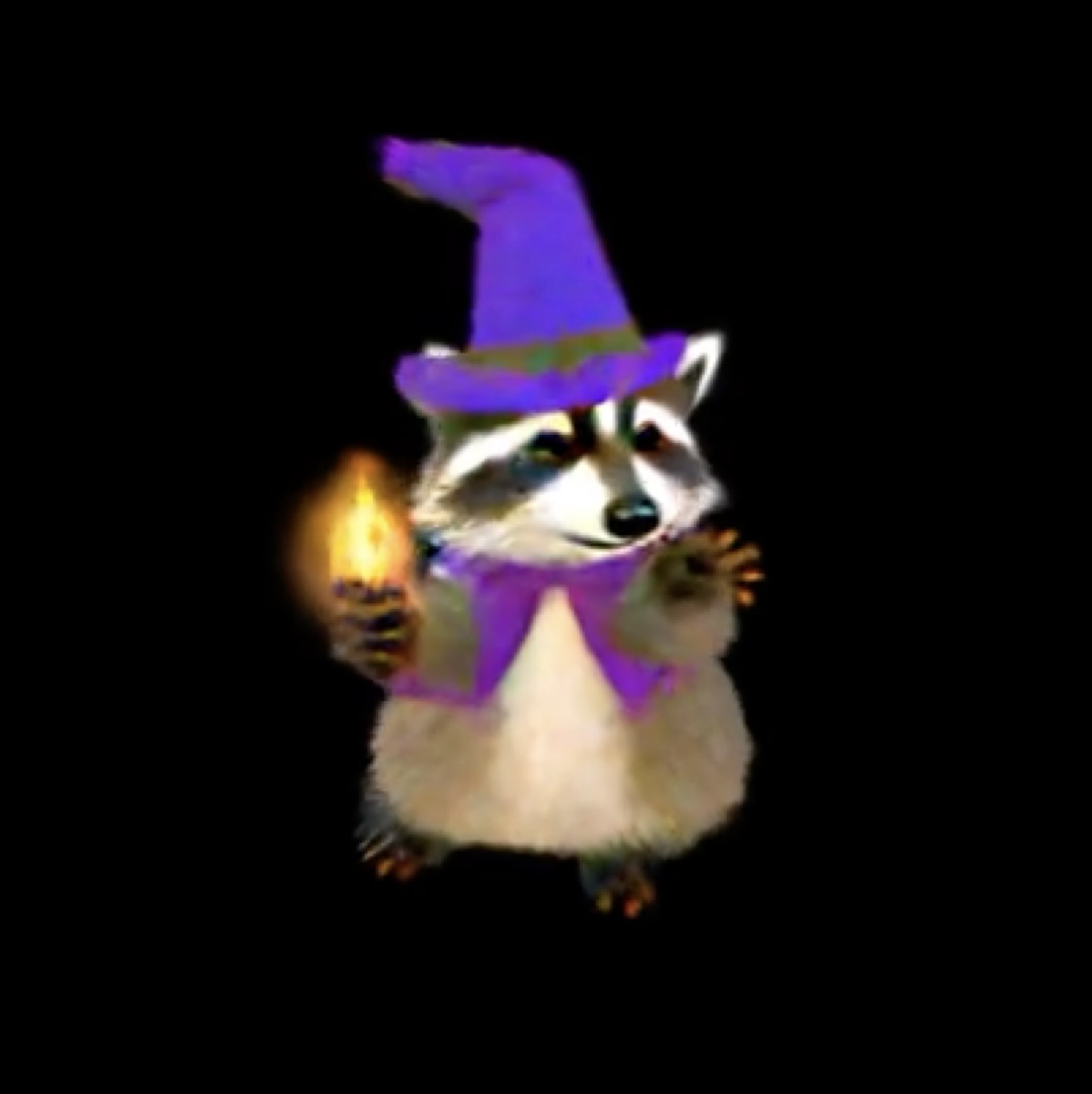}{a wizard raccoon casting a spell}
\end{tabular}
\setcounter{figure}{0}
\vspace{-5pt}
\captionof{figure}{\OURS generates high-quality 3D geometry and appearance, represented as anisotropic 3D Gaussians, from the input text prompts. \OURS reduces the commonly seen multi-face Janus problem, improves local details, and converges robustly without requiring a mesh representation, modifying the SDS loss, or using any additional 3D or multi-view priors.}
\vspace{2em}
\label{fig:teaser}
}
]

\begin{abstract}

In text-to-3D generation, utilizing 2D diffusion models through score distillation sampling (SDS)~\cite{poole2022dreamfusion} frequently leads to issues such as blurred appearances and multi-faced  geometry, primarily due to the intrinsically noisy nature of the SDS loss.
Our analysis identifies the core of these challenges as the interaction among noise levels in the 2D diffusion process, the architecture of the diffusion network, and the 3D model representation. To overcome these limitations, we present \OURS, a methodology incorporating three advances.
First, inspired by InstructNeRF2NeRF~\cite{haque2023instructnerf}, we formalize the equivalence of the SDS generative prior and a simple supervised L2 reconstruction loss. This finding provides a novel tool to debug SDS, which we use to show the impact of time-annealing noise levels on reducing multi-faced geometries. Second, our analysis shows that while image-space diffusion contributes to geometric precision, latent-space diffusion is crucial for vivid color rendition. Based on this observation, \OURS introduces a two-stage training strategy that effectively combines these aspects, resulting in high-fidelity 3D models. Third, we adopt an anisotropic \GS representation, replacing NeRFs, to enhance the overall quality, reduce memory usage during training, and accelerate rendering speeds, and better capture semi-transparent objects. \NAME reduces multi-face geometries, generates fine details, and converges stably.

\end{abstract}    
\section{Introduction}
\label{sec:intro}

Recent advances in Generative AI have marked a paradigm shift across various domains, with notable progress in dialogue generation (\eg, ChatGPT~\cite{chatgpt}), image generation~\cite{ramesh2021dall,rombach2022high,saharia2022photorealistic} and video synthesis~\cite{ho2022imagen,villegas2022phenaki}. However, despite its immense potential, 3D generation still lags behind in these developments. A critical obstacle in 3D generation is the limited size of available datasets, which pale in comparison to the extensive databases used in language~\cite{touvron2023llama} and image fields~\cite{schuhmann2022laion}. To circumvent this lack of 3D datasets, recent efforts such as DreamFusion~\cite{poole2022dreamfusion} leverage 2D text-to-image models by using Score Distillation Sampling to generate 3D models from text prompts, showing exciting results with compelling appearance and geometry. 

However, these text-to-3D approaches are far from  perfect. Several critical issues persist. First, the generated 3D assets frequently exhibit over-saturated colors and blurry appearance. Fine local details are often omitted, giving results a somewhat ``toy-like" quality. Second, the generated 3D asset's geometry tends to be oversimplified, lacking the ability to faithfully represent thin or intricate shapes. Furthermore, these approaches are notorious for exhibiting the ``Janus problem", where the generated 3D object contains multiple canonical views seen from different viewpoints. Lastly, the optimization and rendering speed are hampered by the nature of test-time optimization and the underlying NeRF representation.

In response to the aforementioned challenges, we introduce a simple text-to-3D framework \emph{\NAME}. We start with an empirical analysis that yields two pivotal insights: first, SDS loss can be conceptualized as a supervised reconstruction task using denoised images as ground truth, paving the way for a visualization tool to inspect the training dynamics, and motivating a noise-level annealing to stabilize SDS training. Second, we observe that image-space diffusion excels in geometric accuracy but falls short in color vibrancy. In contrast, latent-space diffusion  enhances color at the expense of geometric fidelity. This leads us to develop a dual-phase training scheme, leveraging distinct diffusion architectures to optimize overall generation quality. Notably, we establish that these observations are agnostic to the underlying 3D representations with broad applicability.
%
A third noteworthy innovation within \NAME is the adoption of \GS~\cite{kerbl3Dgaussians} as the fundamental 3D representation. 
This choice offers a host of distinct advantages, including high fidelity for local details and fast rendering speed. However, directly substituting this representation into existing SDS frameworks leads to low-quality results and artifacts, likely due to the mismatch between noisy SDS loss and the localized nature of \GS. 
To mitigate this, we implement strategies on initialization and density control, achieving a robust convergence to high-quality \GS. 
In summary, our contributions are threefold:
\begin{itemize}
\item Interpreting SDS as a reparametrized supervised reconstruction problem, leading to new visualization that motivates the use of an annealing schedule for noise levels.
\item A two-stage training framework that combines image and latent diffusion for enhanced geometry and color quality.
\item Integration of \GS for text-to-3D generation, with novel regularization techniques for improved quality and convergence, to further improve fidelity and details.
\end{itemize}
With these simple changes, \NAME reduces the multi-face geometry problem and produces a high level of fidelity and details in the synthesized models. \NAME is stable in training, without the need for switching between different 3D representations~\cite{lin2023magic3d}, modification of the SDS loss~\cite{wang2023prolificdreamer}, or additional 3D or multi-view {\em a priori}~\cite{chen2023text,GaussianDreamer}. Our experiments establish \NAME's improvements over leading text-to-3D models.
\section{Related Work}
\label{sec:related}


\paragraph{Text-to-3D.}
Since the advent of large vision-language models~\cite{radford2021learning,rombach2022high,saharia2022photorealistic}, 
the research community has increasingly focused on the generation of 3D assets from textual input.
Early approaches~\cite{Michel_2022_CVPR} utilize the CLIP embedding~\cite{radford2021learning} for alignment between rendered images and text prompts. 
The seminal work DreamFusion~\cite{poole2022dreamfusion} and SJC~\cite{wang2023score} distill the score of learned text-to-image diffusion models~\cite{saharia2022photorealistic,rombach2022high} into optimizing neural 3D models (\eg, NeRF~\cite{mildenhall2021nerf}). These works demonstrate more realistic and high-fidelity results and subsequently became the de facto solutions in this field.

Following the success of DreamFusion/SJC, numerous follow-up works have advanced the field. These approaches encompass a spectrum of ideas including multi-stage refinement~\cite{lin2023magic3d}, geometry and appearance disentanglement~\cite{Chen_2023_ICCV}, and improved the score distillation loss~\cite{wang2023prolificdreamer}. In this work, we study strategies that would enable stable training of a single 3D representation under the SDS framework, without having to convert to meshes (\eg, Magic3D~\cite{lin2023magic3d} and ProlificDreamer~\cite{wang2023prolificdreamer}), designing a different loss (\eg, ProlificDreamer~\cite{wang2023prolificdreamer}, NFSD~\cite{katzir2023noise}), or relying on other 3D or multi-view {\em a priori} that is trained on additional datasets (\eg, GSGEN~\cite{chen2023text}). 

\paragraph{Neural 3D Representations.}

Neural 3D representations originated in the context of 3D reconstruction~\cite{chen2018implicit_decoder, Park_2019_CVPR, Occupancy_Networks}, where neural networks implicitly learned signed distance functions and occupancy functions. This implicit modeling was then extended to the task of novel-view synthesis~\cite{MildenhallECCV2020,Lombardi:2019,sitzmann2019srns}, yielding remarkable rendering outcomes. Subsequent works~\cite{Chen2022ECCV,wang2021neus,yu2020pixelnerf,reiser2021kilonerf} continued refining neural 3D representations from diverse perspectives; for a comprehensive overview, readers are directed to \citet{tewari2022advances}. 
A noteworthy trend~\cite{yu2021plenoctrees,mueller2022instant,sun2021direct} involves the adoption of hybrid implicit-explicit representations, inducing more spatially localized gradient changes for faster training and improved quality.
Most recently, \citet{kerbl3Dgaussians} popularized 3D Gaussians as an innovative, explicit scene representation. In this work, we incorporate a \GS representation and regularized score distillation sampling (SDS) during training. This integration promotes fast convergence and enhances the overall quality of the generated scenes. We diverge in a few details, such as using diffuse color without the spherical harmonics, and we adopt a customized initialization and density control strategy.
Parallel efforts such as GSGEN~\cite{chen2023text}, DreamGaussian~\cite{tang2023dreamgaussian} and GaussianDreamer~\cite{GaussianDreamer} have concurrently chosen \GS as the representation. However, GSGEN~\cite{chen2023text} and GaussianDreamer~\cite{GaussianDreamer} both require an additional 3D prior during training. DreamGaussian~\cite{tang2023dreamgaussian} uses \GS only as coarse initialization for extracting a mesh, whereas we produce high quality \GS directly.

\paragraph{Image Generative Models.}
Generative models for images have been an active area of research, leading to significant advances in the generation of realistic and high-quality 2D content. Early approaches like Variational Autoencoders (VAEs)~\cite{kingma2013auto}, Generative Adversarial Networks (GANs)~\cite{NIPS2014_5ca3e9b1}, and Normalizing Flows~\cite{kobyzev2020normalizing} laid the foundation for this field. In recent years, diffusion models~\cite{sohl2015deep,song2019generative,ho2020denoising} have demonstrated exceptional capabilities in generating large-scale, high-fidelity images with precise textual control over content and style. 
In this work, we aim to ensure a robust and stable training process with the SDS loss. To accomplish this, we incorporate both an image-space diffusion model, DeepFloyd IF~\cite{DeepFloyd}, and a latent-space diffusion model, Stable Diffusion~\cite{rombach2022high}. This strategic combination is employed due to the distinct yet complementary guidance these models offer in the context of text-to-3D generation.
\section{Preliminaries and Notation}
In this section we briefly introduce the background on both Score Distillation Sampling (SDS) and \GS.

\paragraph{Score Distillation Sampling (SDS).} SDS is a loss introduced in DreamFusion~\cite{poole2022dreamfusion} for generating a 3D scene model (such as a NeRF~\cite{mildenhall2021nerf}) from a text prompt $y$  using a pretrained 2D diffusion model. Starting with a randomly initialized scene model, parameterized by $\theta$, we iteratively sample random viewpoints $\pi$ facing the object, and render an RGB image $\vx$ using differentiable rendering, i.e.
$\vx=g(\theta, \pi).$
This rendered RGB image $\vx$ is treated as an image to be denoised with a pretrained 2D diffusion model to obtain an improved image that better aligns with the text prompt. The image $\vx$ is perturbed with additive Gaussian noise $\epsilon \sim \mathcal{N}(0, 1)$ such that 
\begin{equation}
    \label{eqn:diffusion}
    \vx_t = \sqrt{\bar{\alpha}_t}\vx + \sqrt{1-\bar{\alpha}_t}\epsilon, 
\end{equation}
where the noise hyperparameter $t$ determines the magnitude of $\bar{\alpha_t}$, predefined through a fixed variance schedule. The diffusion network typically predicts the added noise $\hat{\epsilon}$. While the diffusion process is iterative, \cref{eqn:diffusion} suggests a one-step prediction of the denoised image as
\begin{equation}
    \label{eqn:denoised_image}
    \hat{\vx}(\vx_t;t,y) = \frac{\vx_t-\sqrt{1-\bar{\alpha}_t}\hat{\epsilon}(\vx_t;t,y)}{\sqrt{\bar{\alpha}_t}}.
\end{equation}
Note that these equations are based on Equations 4 and 15 in  the DDPM paper~\cite{ho2020denoising}.

The DreamFusion authors find that omitting the poorly conditioned diffusion network Jacobian term from the typical diffusion training loss gradient gives a more stable gradient for backprogation to the current scene model, resulting in the SDS loss gradient 
\begin{equation}
    \label{eq:sds_loss}
    \nabla_\theta \ell_\text{SDS} \left(\vx=g(\theta, \pi)\right) \triangleq \E_{t,\epsilon} \left[w_t \left(\hat{\epsilon}(\vx_t; y, t)-\epsilon \right)\frac{\partial \vx}{\partial \theta} \right]\,.
\end{equation}
In DreamFusion, this is shown to be the gradient of a weighted probability density distillation loss. In \cref{sec:approach:sds-analysis}, we explore a more intuitive interpretation of the SDS loss that leads to a natural tool for visualization. 

    

\paragraph{3D Gaussian Splatting} \GS is an explicit 3D representation popularized by \cite{kerbl3Dgaussians}, where the scene is comprised of a large set of semitransparent anisotropic 3D Gaussians. These Gaussian primitives are geometrically parameterized by covariance (or equivalently scale and rotation) and position, with appearance parameterized by color and opacity. This representation has been shown to achieve remarkable results in the area of novel-view synthesis, with significantly higher quality and rendering speed compared to previous volumetric methods based on radiance fields.
To render 3D Gaussians, each primitive is projected into a screen space 2D Gaussian and sequentially rasterized in a back-to-front manner using alpha-blending. For screen-space positions $\mu_i$, screen-space covariances $\Sigma_i$, colors $c_i$, and opacities $\sigma_i$, the per-primitive alpha values and the final composited rendered color at pixel position $x$ are
\begin{align*}
\alpha_i(x) &= \sigma_i e^{-\frac{1}{2}(x-\mu_i)^T\Sigma_i^{-1}(x-\mu_i)}\\  C(x) &= \sum_i{c_i\alpha_i(x)\prod_{j<i}{(1-\alpha_j(x))}}
\end{align*}
This rendering process is fully differentiable (given a differentiable sorting subroutine), enabling its use as a representation for text-to-3D generation.
\begin{figure*}[hbtp]
  \includegraphics[width=\textwidth]{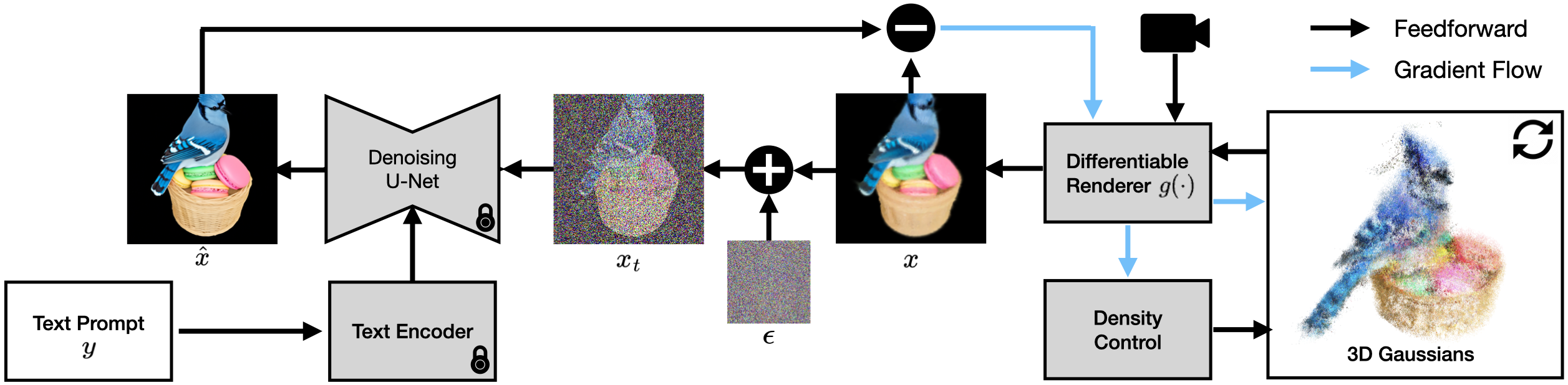}
  \caption{Our pipeline, \OURS, is an iterative optimization framework for creating anisotropic \GS from text prompts. It begins with a text prompt as input, which is then processed by a fixed, pretrained text encoder to generate a text embedding. This embedding serves as conditioning input for our pretrained diffusing U-nets. During each iteration, we randomly sample a viewpoint and render the 3D Gaussians into an RGB image , which is subsequently input into the U-net for denoising and enhancement. The discrepancies between the denoised images and the originally rendered images are utilized as gradients for updating the anisotropic \GS.}
  \label{fig:pipeline}
\end{figure*}

\section{\OURS}
\label{sec:approach}
%

In a nutshell, \OURS addresses both the common blurry appearance and the multi-face geometry problems in SDS training with three conceptually simple modifications: (1) time-annealing of noise levels for 2D diffusion, which reduces multi-face geometries; (2) a dual-phase training that utilizes image-space diffusion for accurate geometry and subsequently a latent-space diffusion for vibrant and sharp appearances; and (3) integration of \GS with regularization and density control that aims to improve model capacity for local details and transparent objects. 

\subsection{Inspecting and Taming SDS Loss}
\label{sec:approach:sds-analysis}

A key challenge of optimization with the SDS loss is the noisy gradients inherent in the formulation. To address this, we first propose a novel interpretation that links it to NeRF reconstruction (specifically, Instruct-NeRF2NeRF~\cite{haque2023instructnerf}). This theoretical connection leads to two practical benefits: an annealing strategy for noise levels to improve convergence and a new visualization tool for inspecting the training dynamics of SDS. 

\paragraph{The SDS Generative Prior and NeRF Reconstruction.} In the DreamFusion training paradigm, the 3D scene representation is treated as an image generator while the SDS loss is treated as a prior over the generated images. While this probability-based interpretation allows the use of statistical tools (e.g.\ \cite{wang2023prolificdreamer}), a more practical lens is suggested in a different related work.
Instruct-NeRF2NeRF~\cite{haque2023instructnerf} is a recent work that also uses generative 2D models, albeit for a style transfer application rather than text-to-3D generation. In this work, the usual supervised reconstruction framework is used where a set of ground truth images is compared against a rendering from the current scene model. During training, Instruct-NeRF2NeRF uses the generative model to iteratively replace individual ground truth images with results from the 2D image generator (which may not be multiview-consistent) based on the current rendering result from that viewpoint. The authors note that their training process can be interpreted as a variant of SDS. 
Here we make this connection explicit: 
\begin{prop}\label{prop:main}
Training a 3D scene representation with the SDS generative prior is mathematically equivalent (up to scale) to using L2 reconstruction loss against images generated from the 2D generator.
\end{prop}
\begin{proof}
Without loss of generality, consider the SDS loss with an image-space diffusion model without classifier-free guidance. 
We use \cref{eqn:denoised_image,eqn:diffusion} to expand the noise residual:
\begin{equation*} 
\begin{split}
\hat{\epsilon}(\vx_t;t,y) - \epsilon &= \frac{\vx_t - \sqrt{\bar{\alpha_t}} \hat{\vx}(\vx_t;y,t)}{\sqrt{1-\bar{\alpha_t}}}-\epsilon \\
&=\frac{\sqrt{\bar{\alpha_t}}\vx +\sqrt{1-\bar{\alpha_t}}\epsilon - \sqrt{\bar{\alpha_t}} \hat{\vx}(\vx_t;y,t)}{\sqrt{1-\bar{\alpha_t}}}-\epsilon \\
&= \frac{\sqrt{\bar{\alpha_t}}}{\sqrt{1 - \bar{\alpha_t}}}\left(\vx - \hat{\vx}(\vx_t;y,t)\right) \\
\end{split}
\end{equation*}
Then, the gradient of the SDS loss is implemented as
\begin{equation*} 
\label{eq1}
\begin{split}
    \nabla_\theta \ell_\text{SDS} \left(\vx=g(\theta, \pi)\right)
    & \triangleq w(t) \left(\hat{\epsilon}(\vx_t; y, t)-\epsilon \right)\frac{\partial \vx}{\partial \theta} \\
     &= w(t) \frac{\sqrt{\bar{\alpha_t}}}{\sqrt{1-\bar{\alpha_t}}}\left(\vx - \hat{\vx}(\vx_t;y,t)  \right)\frac{\partial \vx}{\partial \theta},
\end{split}
\end{equation*}
which is exactly the gradient of a scaled L2 loss $\ell_{L2}(\vx, \hat{\vx}) = \frac{\beta(t) }{2}\| \vx - \hat{\vx}\|^2$ between the current rendering $\vx$ and ground truth image $\hat{\vx}(\vx_t; y,t)$, with $\beta(t) = \frac{w(t)\sqrt{\bar{\alpha_t}}}{\sqrt{1-\bar{\alpha}_t}}$. For latent-space diffusion models, a similar line of reasoning shows that SDS loss is instead equivalent to a latent-space L2 loss.
\end{proof}

\paragraph{Annealing of Noise Level.}
The above discussion establishes a novel perspective where the one-step denoised image $\hat{\vx}$, as defined in \cref{eqn:denoised_image}, is conceptualized as the ground truth image in the context of NeRF reconstruction. This insight yields significant implications for the noise level scheduling in the 2D diffusion process. Particularly, to ensure effective convergence during SDS training, it is crucial that the variance of these ground truth images starts large and decreases as training advances. To achieve this, we dynamically adjust the noise distribution's upper and lower limits, progressively narrowing the range with training iterations. We use a piecewise linear schedule for the upper and lower bounds that converge by the end of the training. Guiding this noise magnitude is critical, since excessive noise leads to larger gradient magnitudes (equivalent to having a changing ground truth), which can lead to worse model convergence as shown later in \cref{sec:exp-ablation-sds}. Incidentally, ProlificDreamer~\cite{wang2023prolificdreamer} proposes a similar but simpler annealing strategy, reducing noise level after initial iteration steps. 

\paragraph{Visualization of Supervision Signals.} 
A second advantage of implementing the proposed SDS loss reparameterization lies in the enhanced interpretability of the training process. Through the visualization of the pseudo-ground-truth image $\hat{\vx}$ throughout the training phase, we gain insights into the direct influence of different hyperparameters on target images. This capability empowers us to devise a more robust training scheme, effectively taming the inherent noise in SDS loss for text-to-3D tasks.

A common challenge for 3D generation from text is the tendency for these systems to form objects with multiple faces. By examining the latent images we find a relationship between the multi-face problem and the SDS noise parameter. \Cref{fig:sds_noise_annealing_latent_viz} shows the predicted original images $\hat{\vx}$ from two training runs with different noise levels. For the run with larger noise the system is more likely to hallucinate a face on the back of the dog's head. Since each iteration is conditioned on the previous state, repeated selection of large noise values can cause the model to converge to a geometry with many faces. On the flip side, using lower noise levels reduces the signal to the optimization as the latent images do not change between iterations. Taken together, these results suggest we should use an annealing strategy for the added noise where it begins with a larger range and narrows as the training progresses. 

\begin{figure*}[hbtp]
  \centering
  \includegraphics[width=0.70\linewidth]{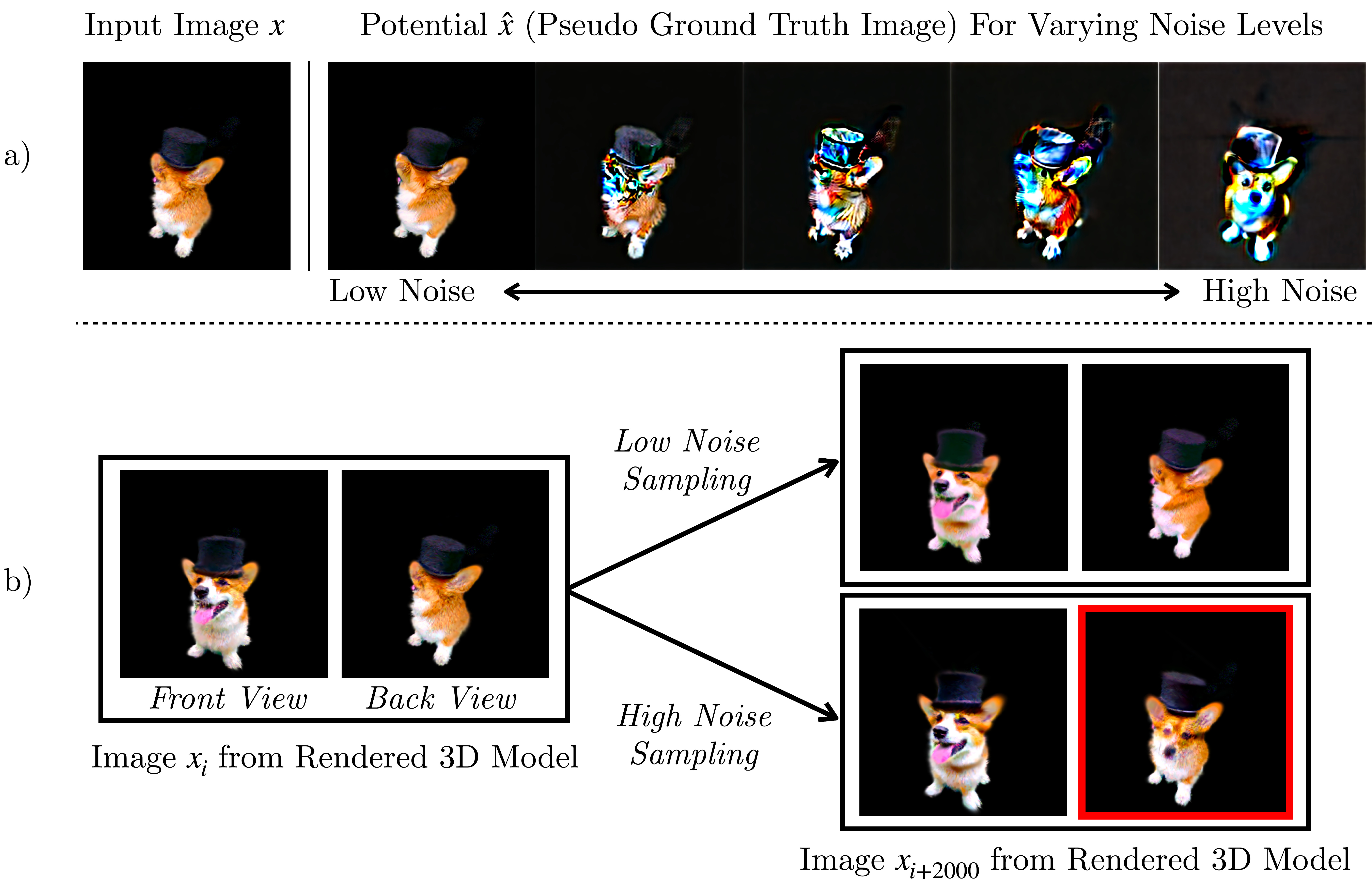}
  \vspace{-3pt}
  \caption{a): Per Proposition \ref{prop:main}, the reformulated loss equation enables visualization of the one-step denoised image $\hat{\vx}$ which allows us to observe the effect of modifying the level of noise being injected into $\vx_t$ in \cref{eqn:diffusion} and subsequently $\hat{\vx}$ in \cref{eqn:denoised_image}. Less noise produces images closer to the input image $\vx$ while larger noise levels produce more variation.\\
b): Two training runs are compared, one biased to sample lower noise (top) and one biased to sample higher noise (bottom). Two views are rendered at both an early iteration $i$ and later iteration $i+2000$. From a), high noise samples are associated with a face incorrectly hallucinated on the back of the dogs head. Unsurprisingly, the model with larger noise ends up converging to a multi-faced dog.}
  \label{fig:sds_noise_annealing_latent_viz}
  \vspace{-5pt}
\end{figure*}

Similarly, the visualizations of the one-step denoised image $\hat{\vx}$ for various guidance scales in \cref{fig:guidance_scale_latent_viz} provide insight into the effect of the guidance scale hyperparameter. Lower values lead to smooth images lacking fine details, while larger values hallucinate high-frequency details and over-saturated colors. This can lead to fake-looking images as shown in \cref{sec:exp-ablation-sds}. While the effect this parameter is already understood, this simple example highlights the insights made possible by this reparameterization. 

\begin{figure}[hbtp]
  \centering
  \includegraphics[width=\linewidth]{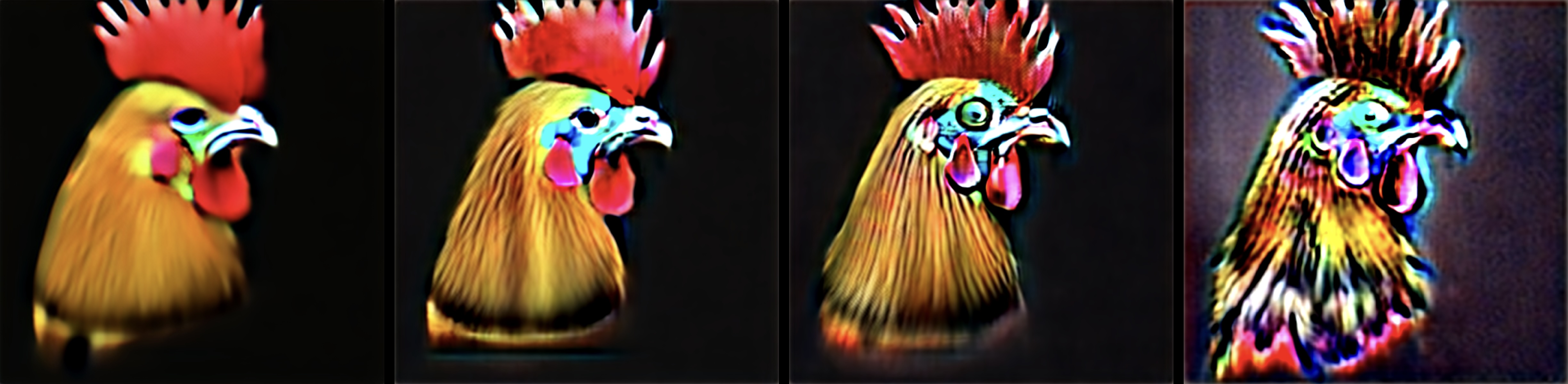}
  \vspace{-3pt}
  \caption{Understanding the impact of guidance scale on the appearance via visualizing the one-step denoised images $\hat{\vx}$ during training. Left-to-right: Guidance scale 10, 20, 35, and 100. As the guidance scale increases, so does the high frequency detail and color, eventually leading to an unrealistic image. }
  \label{fig:guidance_scale_latent_viz}
  \vspace{-6pt}
\end{figure}



\begin{figure}[htbp]
  \includegraphics[width=\linewidth]{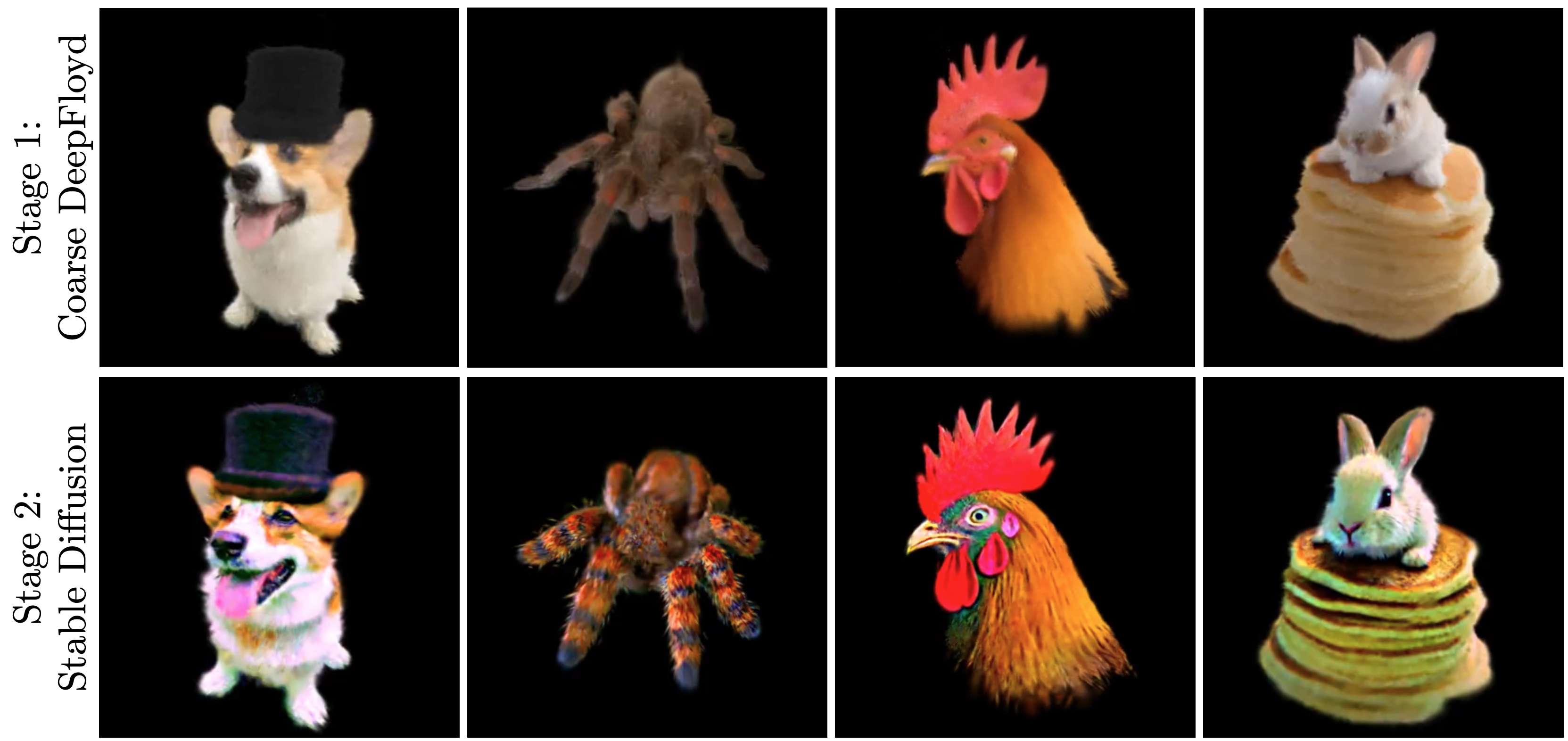}
  \caption{Results from two training stages. Stage 1 (top): image-space diffusion (DeepFloyd) produces accurate geometry at the cost of muted colors. Stage 2 (bottom): we finetune with latent-space diffusion (Stable Diffusion) to enhance the appearance.}
  \label{fig:ablation_gs_if_sd}
  \vspace{-5pt}
\end{figure}

\subsection{A Tale of Two Diffusions: Image vs. Latent}
The current landscape of diffusion models in the literature bifurcates into two categories: image-space diffusion and latent-space diffusion. 
Image-space models, such as DeepFloyd~\cite{DeepFloyd} and Imagen~\cite{saharia2022photorealistic}, directly apply noise to the images. 
In contrast, latent-space models like Stable Diffusion~\cite{rombach2022high,podell2023} necessitate an encoder-decoder pair that maps between the image and latent spaces, applying noise only in the latent domain. 
Our empirical analysis reveals that these two model types exhibit different guidance directions for text-to-3D. 
We propose an effective two-stage training framework that leverages their distinct properties. 
As shown in~\cref{fig:ablation_gs_if_sd}, the proposed framework can produce sharp texture and detailed geometry. Incidentally, Magic3D~\cite{lin2023magic3d} arrives at a similar training strategy but mainly for reasons of speed and resolution, rather than quality.

\paragraph{Image-space diffusion for geometry reconstruction.} 
For the first stage of training, we propose to use the image-space model, DeepFloyd~\cite{DeepFloyd}, to train the 3D model.
The primary goal at this stage is to converge to a reasonable rough geometry, so that a detailed appearance can be learned later in the optimization, as shown in the first row of~\cref{fig:ablation_gs_if_sd}.
Therefore, in this stage, we only use the coarse DeepFloyd model, operating at $64 \times 64$ resolution.
At this stage, all the parameters of the 3D models are learnable. 
A low learning rate is used for the geometry as it converges (see \ref{app:ablation-density-control} for more detailed analysis). 

\paragraph{Latent-space diffusion for appearance enhancement.}
While the coarse reconstruction successfully yields a 3D model with satisfactory geometric accuracy, it tends to fall short in terms of visual quality due to its use of low-resolution 2D image supervision at 64 x 64 resolution. 
The primary objective of the refinement stage is to significantly enhance the visual fidelity of the 3D model, as shown in the second row of~\cref{fig:ablation_gs_if_sd}.
To achieve this, we employ a latent-space diffusion model, Stable Diffusion (SDv2.1-base)~\cite{rombach2022high} trained with $512 \times 512$ resolution images.
As shown in~\ref{app:ablation-two-stage-training}, the image-space diffusion models are not suitable to get the detailed appearance for the 3D model (even for a high-resolution model like DeepFloyd with super-resolution modules).
We hypothesize that this is due to view-inconsistent pixel-level guidance, resulting in a blurred model and the loss of appearance detail. 
In contrast, the guidance from the latent-space diffusion model is less sensitive to this issue, since the loss is calculated in the latent space after feature compression from the image encoder. 
As a result, with the guidance from Stable Diffusion at the second stage, we largely increase model fidelity for both appearance and geometry.

\subsection{Integrating 3D Gaussians}
The aforemention training scheme provides stablized training with NeRF, yet there is potential for further enhancement in the finer details. \GS offer advantages such as rapid rendering speeds and enhanced local representation over other NeRF representations. However, they are sensitive to the hyper-parameters and training strategies. In fact, directly substituting this representation into our existing training frameworks leads to low-quality results and artifacts, likely due to the mismatch between noisy SDS loss and the localized nature of \GS. Specifically, we observe that despite having on average 10x larger gradient magnitude compared to other learnable parameters (\eg, colors, scales, rotation), the position variables exhibit a "random walk" behavior without converging to a high-quality geometry. This observation motivates specialized \GS training strategies around initialization and density control.


\paragraph{Initialization.} In 3DGS~\cite{kerbl3Dgaussians}, Structure-from-Motion (SfM) is used to initialize the Gaussian locations for scene reconstruction. However, this method cannot be used in text-to-3D generation. Thus, we use a simple alternate approach that has proved compatible with a wide range of text prompts. To start, the centers of the Gaussian primitives are randomly sampled with a uniform distribution over a volume. While the positions are uniformly distributed, the opacity of each point is initialized relative to its proximity to the center of the volume. More specifically, the initial opacity linearly decays with distance from the origin. This simple heuristic helps with convergence since the majority of generated objects have most of their density closer to the center of the scene. 
\paragraph{Density control.} 
Our experiments show that the position learning of \GS is hard to tune and easily diverges with large learning rates due to the noisy signal from SDS loss. To stabilize training, a small learning rate is required for the position variables to avoid moving too far from their initial locations. Consequently, we cannot solely rely on position learning to produce fine geometry.
Therefore, we turn to density control for geometry construction. Specifically, after initialization, we apply periodic densification and pruning, gradually adding new points in order to produce finer geometry and appearance.  Additionally, we find that resetting the opacities to near zero at early training stages helps reduce floaters and bad geometry. Please refer to \ref{app:density-control-details} for details of our implementation.
\section{Experiments}
\label{sec:experiment}
We compare \OURS against several state-of-the-art text-to-3d methods on the overall quality of the synthesized 3D geometry and appearance as well as memory usage during training and rendering speed. More ablation studies can be found in our appendices.

\subsection{Comparison To Prior Methods}
As shown in ~\cref{fig:sota}, \OURS achieves state-of-the-art results compared to baseline works including DreamFusion~\cite{poole2022dreamfusion}, Magic3D~\cite{lin2023magic3d}, GSGen~\cite{chen2023text}, and ProlificDreamer~\cite{wang2023prolificdreamer}. \OURS's initial coarse geometric optimization converges to accurate geometry, greatly reducing the occurrence of multi-faced geometry commonly seen in the baseline methods. \cref{tab:speed} presents an efficiency analysis of our method in comparison to baseline approaches. Our method, employing \GS, renders at $>30$FPS while maintaining reasonable training time and minimal GPU memory usage. 
Notably, Magic3D tends to produce over-saturated color while ProlificDreamer and GSGen achieve similar detailed textures but consistently produce multi-faced or otherwise incorrect geometries (additional visualization in \ref{app:additional_vis}).  


\begin{figure*}[htp]
  \centering
  \includegraphics[width=0.93\textwidth]{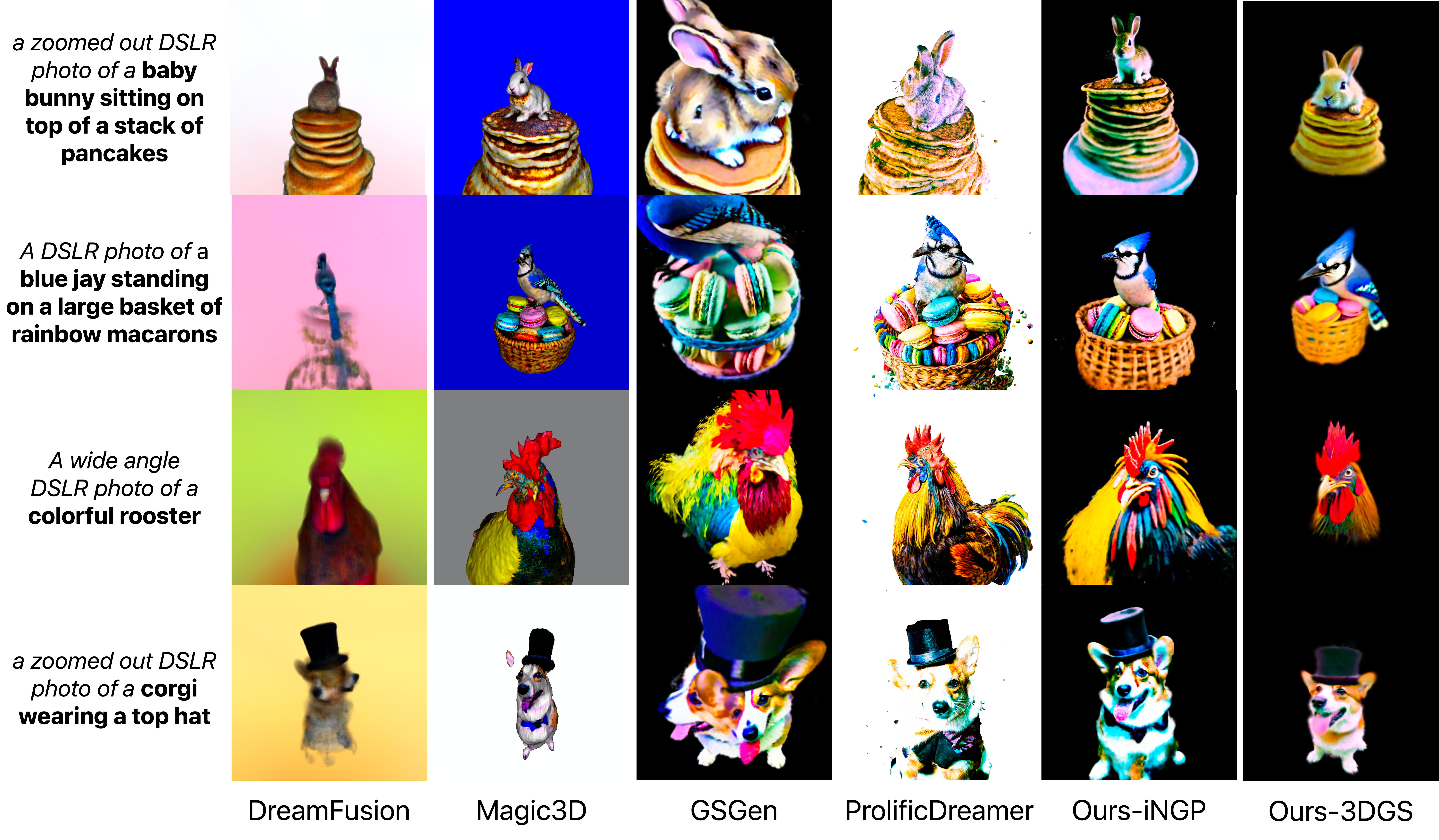}
  \vskip -0.1in
  \caption{Comparison against prior methods. Prior methods typically have problems such as blurriness (DreamFusion~\cite{poole2022dreamfusion}), multi-face geometry (Magic3D~\cite{lin2023magic3d}, GSGen~\cite{chen2023text}, and ProlificDreamer), over-saturation in color (Magic3D~\cite{lin2023magic3d}), cartoony appearances, or mismatch between content and text prompts. \OURS (including both iNGP~\cite{mueller2022instant} and \GS~\cite{kerbl3Dgaussians} geometry primitives) achieves accurate geometry representation with fine details while preserving a realistic appearance. Results for DreamFusion and Magic3D use the open-source Threestudio implementation~\cite{threestudio2023} since the authors have not released their code. Additional visualization are shown in \ref{app:additional_vis}.
  }
  \vskip -0.1in
  \label{fig:sota}
\end{figure*}



\subsection{Generalization Across 3D Representations} \label{sec:ablation-gs}
We showcase the efficacy of \GS compared to volumetric radiance fields, specifically iNGP~\cite{mueller2022instant}. iNGP \cite{mueller2022instant} was widely adopted in previous work~\cite{Chen_2023_ICCV,lin2023magic3d,wang2023prolificdreamer} thanks to its speed compared to classical MLP-based implicit neural representations~\cite{poole2022dreamfusion}. To ensure an equitable evaluation, both \GS and iNGP were trained with the proposed training scheme. The qualitative results are reported in the two rightmost columns in \cref{fig:sota}. Our training scheme is generalizable beyond \GS and works well on iNGP. Overall, \GS still produce better local details than iNGP, supporting our choice of 3D representation.
For detailed structures (\eg hairs from corgi and bunny), iNGP typically produces either blurry or noisy surface textures, while \GS generate realistic detailed structures. iNGP also results in temporal aliasing and flickering, which is visible only in videos.

Quantitative efficiency measurements, presented in \cref{tab:speed}, indicate the advantages of \GS. With a similar parameter count, \GS utilize $82\%$ less GPU memory and render 6 times faster faster than iNGP~\cite{mueller2022instant}. Interestingly, training time between the two methods remained comparable, largely owing to the fact that the 2D diffusion models constitute the dominant time-consuming component in the forward process, especially in the coarse stage when rendering resolution is low.


\begin{table}[htbp]
    
    \begin{center}
    \begin{adjustbox}{width=1\linewidth}
    \begin{small}
    {
        \begin{tabular}{cccccc}
        \toprule
        
        & \textbf{Training Time} & \textbf{Peak Memory Usage} & \textbf{Render Speed}\\
         & \textbf{(min)} & \textbf{(GB)} & \textbf{(fps)}\\
        \midrule
        DreamFusion-iNGP (12.6M) \cite{poole2022dreamfusion} & \textbf{40} & 17.6 & 14.0 \\ 
        Magic3D (12.6M) \cite{lin2023magic3d} & 75 & 16.6 & 9.4 \\ 
        ProlificDreamer (12.6M) \cite{wang2023prolificdreamer} & 277 & 31.8 & 10.8 \\ 
        GSGen (4.7M) \cite{chen2023text} & 228 & 9.9 & \textbf{52.5}\\ 
        \midrule
        Ours-iNGP (12.6M) & 81 & 31.9 & 7.38 \\ 
        Ours-3DGS (14M) & 97 & \textbf{5.7} & 46.0\\ 
        \bottomrule
        \end{tabular}
    }
    \end{small}
    \end{adjustbox}
\end{center}
\vskip -0.2in
\caption{Comparison of parameter count, training time, memory usage, and render speed. The evaluations are performed on a single NVIDIA V100 GPU. DreamFusion and Magic3D are not open-sourced so we use the Threestudio implementation~\cite{threestudio2023}.}
\vskip -0.1in
\label{tab:speed}
\end{table}


\subsection{Ablation on SDS Annealing} 
\label{sec:exp-ablation-sds}

A critical aspect of the optimization processes described in \cref{fig:pipeline} is the addition of noise to the image generated by the 2D diffusion model. Noisy gradients are a common issue with SDS loss and, as shown in \cref{sec:approach:sds-analysis}, crafting a schedule for the noise bounds is important for consistently converging to good results. Our results shown in \cref{fig:annealing} match what we find in our analysis of the visualizations of the one-step denoised images and demonstrate that high noise levels during training tend to produce artifacts and multi-faced geometry. Intuitively, as the model converges, less noise should be added each step once the optimization has settled into a single local minimum. 



\begin{figure}[htbp]
\centering
  \includegraphics[width=0.85\linewidth]{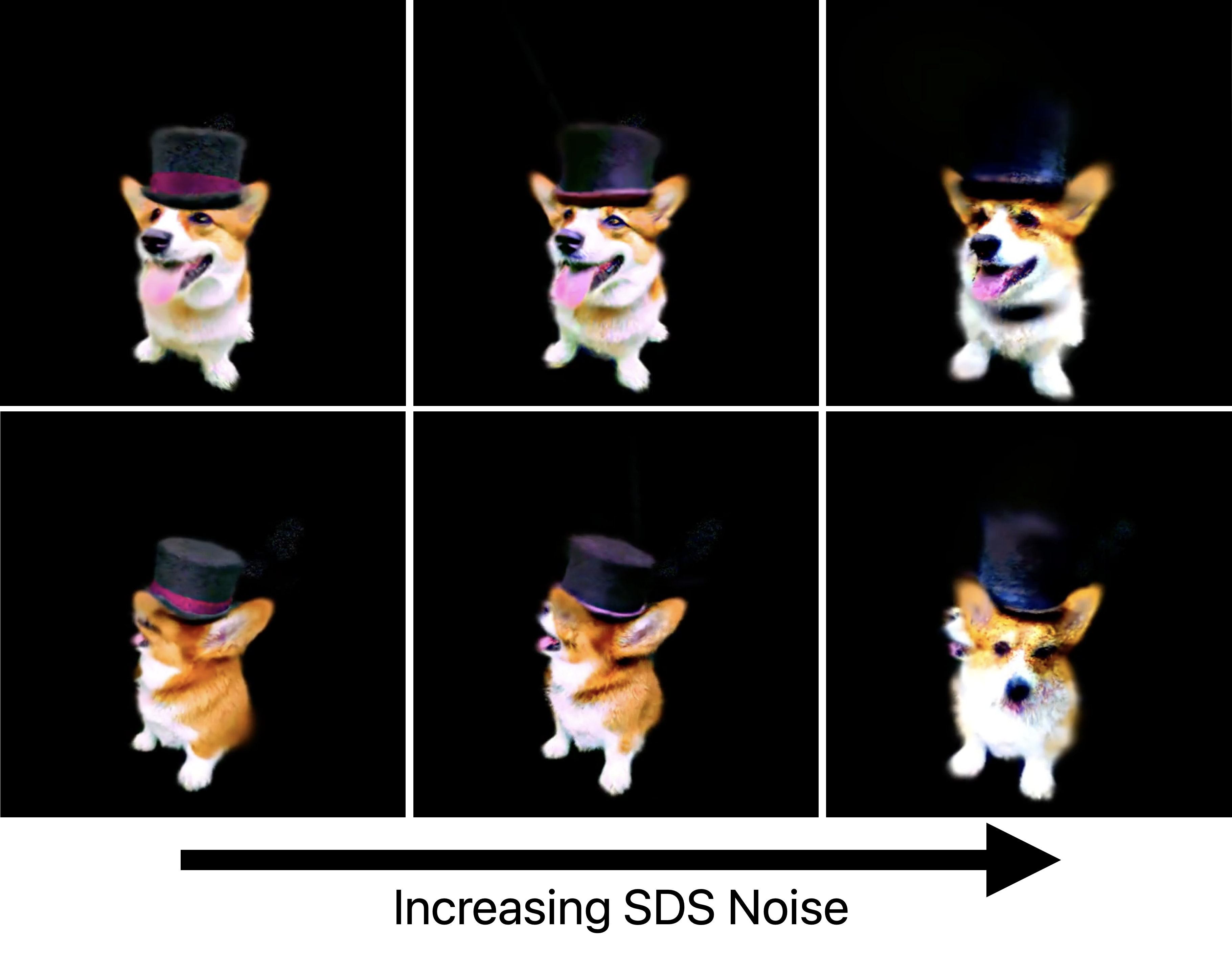}
  \vspace{-10pt}
  \caption{The upper and lower bounds of the noise being injected into $\vx_t$ in \cref{eqn:diffusion} change as a function of the training iteration. Larger noise levels give more high-frequency texture detail, but also more artifacts including multiple faces. The converged model is shown from the front (top row) and back (bottom row), with increasing levels of noise left-to-right. }
  \label{fig:annealing}
  \vspace{-0.2in}
\end{figure}



\section{Failure Analysis}
\label{sec:failure}

\begin{figure}[htbp]
  \includegraphics[width=\linewidth]{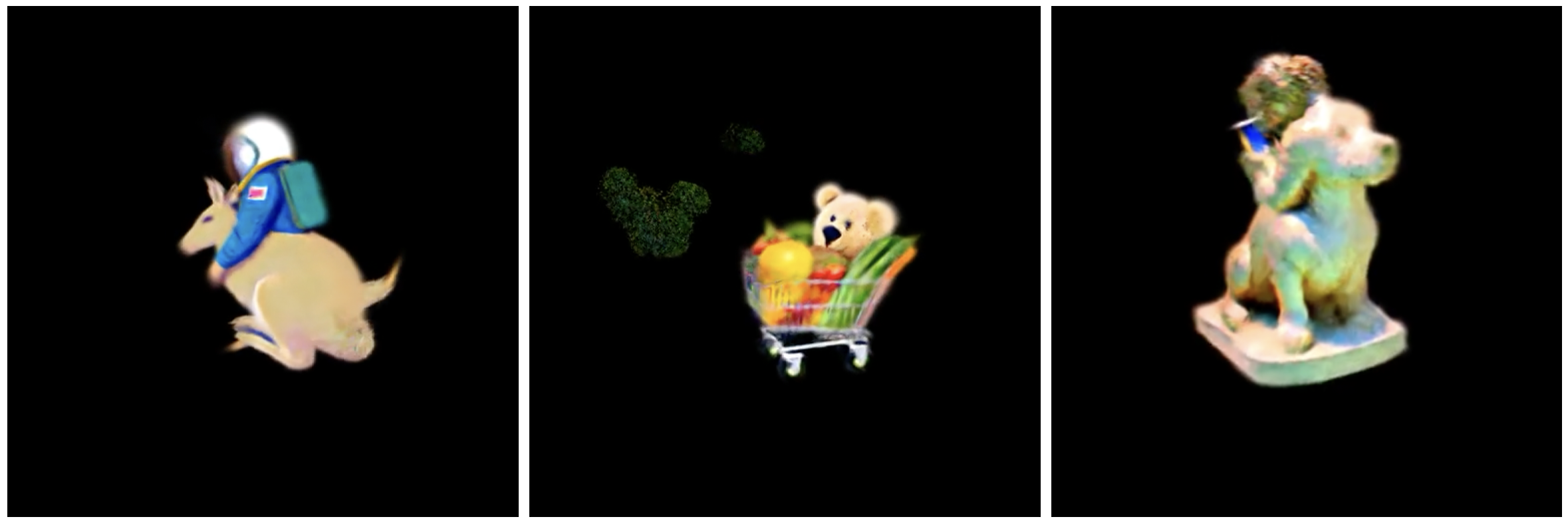}
  \vspace{-15pt}
  \caption{Failure cases: ``An astronaut riding a kangaroo" with the astronaut being erroneously merged in; ``A teddy bear pushing a shopping cart full of fruits and vegetables" with floaters; and ``Michelangelo style statue of dog reading news on a cellphone" with multi-face and blurry geometries.}
  \label{fig:failure_case}
  \vspace{-0.2in}
\end{figure}

While our strategies are shown to reduce multi-face geometry, there remain scenarios where these methods do not yield satisfactory results, as illustrated in \cref{fig:failure_case}. For instance, some failures originate from the 2D diffusion model's inability to accurately interpret the prompt, while others produce floating or blurry geometries. Multi-face geometry also still exists for certain prompts.

\section{Conclusion}
\label{sec:conclusion}
In this work, we introduce \OURS, a text-to-3D framework that addresses the blurry appearance and multi-faced geometry problems that are commonly seen in prior methods.  Our analysis reveals that the Score Distillation Sampling loss can be reparametrized as a supervised reconstruction loss using denoised images as pseudo-ground-truth. This finding leads to intuitive ways to visually inspect the training dynamics and the formulation noise level annealing strategies that reduce the occurrence of multi-face artifacts. Empirical results show that image-space diffusion assists in generating better geometry while latent-space diffusion produces vibrant and detailed colors, inspiring our dual-phase training scheme. Notably, both the reparametrization and training schemes are agnostic to the underlying 3D representations and generalize beyond \GS.  However, to enhance detail and construction fidelity, we adopt a \GS as our core 3D representation, including a number of strategies involving initialization and density control to enhance the robustness and convergence speed toward accurate geometric representations. Our empirical study demonstrates the superior quality of our method in comparison to previous approaches.

\clearpage
{
    \small
    \bibliographystyle{ieeenat_fullname}
    \bibliography{main}
}

\clearpage
\clearpage
\setcounter{page}{1}
\setcounter{section}{0}
\setcounter{figure}{0}
\renewcommand{\thesection}{Appendix \Alph{section}}
\renewcommand{\thefigure}{app-\arabic{figure}}

\maketitlesupplementary

\section{Additional Visualization}
\label{app:additional_vis}
\cref{fig:multiview_comparison_big} shows additional result comparison with different view of angles.
\OURS is able to generate the 3D model with both detailed texture and geometry compared to the baseline methods.

\begin{figure*}[hbtp]
  \centering
  \includegraphics[width=\textwidth]{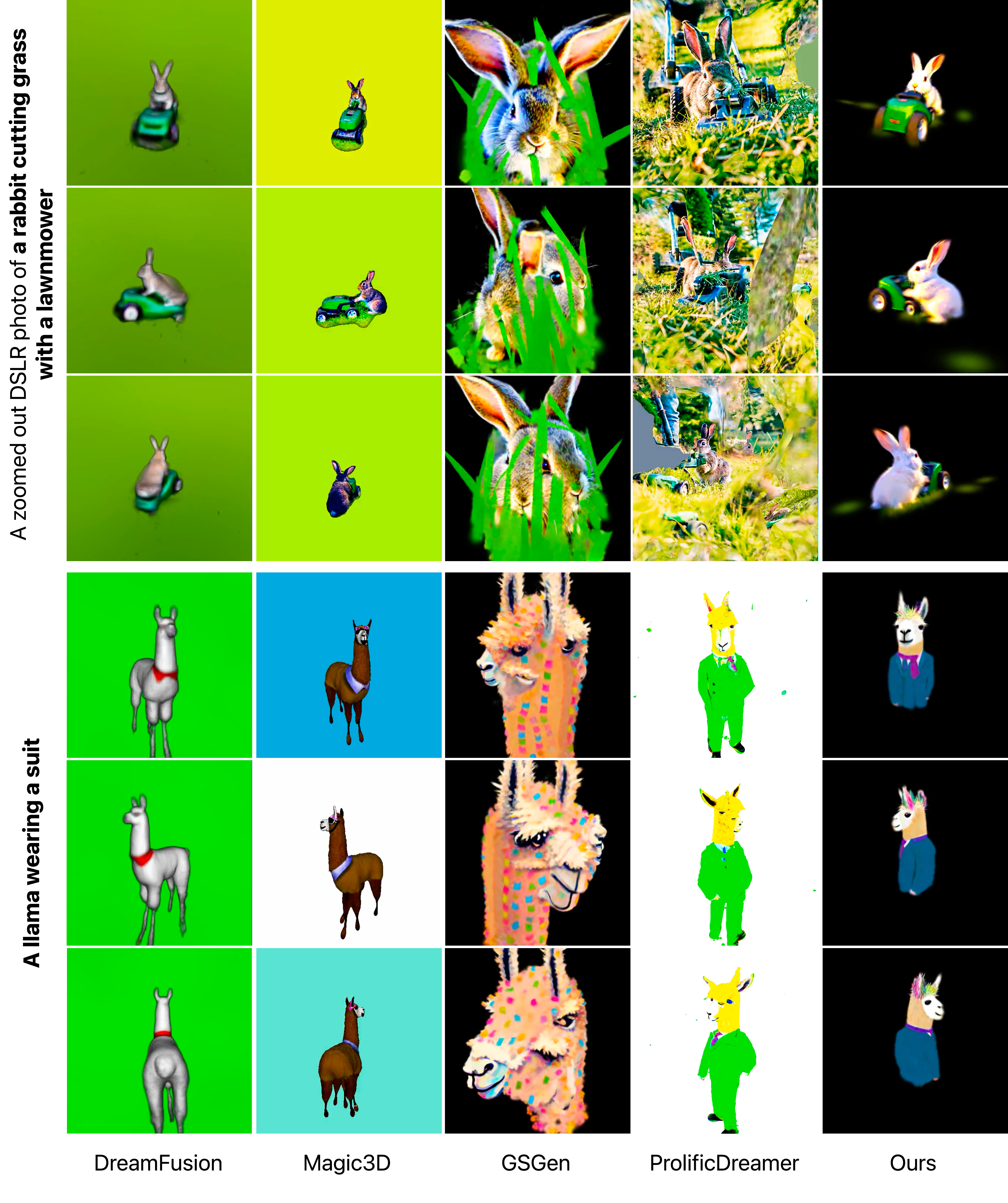} 
  \caption{Multi-view comparison against prior methods. Each column shows the generated object from 3 different views roughly equally spaced about the vertical axis. GSGen and ProlificDreamer struggle to produce 3D view-consistent geometry. DreamFusion and Magic3D do not have released code so we use the open-source Threestudio implementation~\cite{threestudio2023}. }
  \label{fig:multiview_comparison_big}
\end{figure*}




\section{Density Control Setup} \label{app:density-control-details}
\cref{fig:stage-1-schedule} shows an illustration of our density control setup.
To assist with the convergence of the geometry of the scenes, we use the following schedule to modify the \GS. 
Firstly, we randomly initialize 1000 points based on the aforementioned initialization scheme. 
As shown in \ref{app:ablation-density-control}, we intend to use less starting points to reduce the floaters and produce better geometry. 
Then, for every 500 iterations we apply a densification process based on the original Gaussian splatting method~\cite{kerbl3Dgaussians}. 
More specifically, we split and clone the Gaussians when the magnitude of the position gradient is over a threshold. 
By doing so, we can allow the representation to better capture fine details. 
Please refer to the original paper~\cite{kerbl3Dgaussians} for more details of the densification algorithm. 
Note that we start this densification process after 100 iterations. 
This is to make sure the averaged positional gradients get stabilized. 
Similar to the original method, we also apply periodic pruning immediately after densification to remove the Guassians with smaller opacities or large 2D projected area. 
In addition, as shown in the ablation study in \ref{app:exp-ablation-density-control}, we found that resetting the opacities at the early training stage can help to reduce the floaters in the final result. 
In our setup, we choose to reset the opacities at the 1000th iteration.
This is due to the positions and other attributes of the primitives have begin to converge before 1000 iteration, and resetting this parameters allows for a more robust convergence by preventing the optimization from getting caught in the initial local minima (e.g., floaters or bad geometry). 
The density control process ends at 12000 iterations; we then proceed with 3000 fine-tuning iterations with a fixed number of \GS to smooth out the spiky artifacts introduced by densification.

\section{Ablation on Density Control} \label{app:ablation-density-control}
As shown in Figure~\ref{fig:stage-1-schedule}, to assist with the convergence of the geometry of the scenes, we use the following schedule to modify the \GS. Firstly, we randomly initialize 1000 points based on the aforementioned initialization scheme. Then, every 500 iterations we apply a densification process based on the original Gaussian splatting method~\cite{kerbl3Dgaussians}. More specifically, we split and clone the Gaussians when the magnitude of the position gradient is over a threshold. By doing so, we can allow the representation to better capture fine details. Please refer to the original paper~\cite{kerbl3Dgaussians} for more details of the densification algorithm. Note that we start this densification process after 100 iterations. This is to make sure the averaged positional gradients get stabilized.

\paragraph{Initialization.} As shown in \cref{fig:init}, starting with fewer points and annealing the initial opacity of the Gaussians results in the best geometry.
More specifically, comparing the results from the same row, the results with opacity decay in the right column (\ie, linearly decaying opacity based on the distance to the origin) have less floaters. 
Furthermore, comparing the results from the same column, with more starting points (from top to bottom), there are more floaters and the training become unstable if we initialize with a large amount of points due to the noisy signal from SDS loss (see the figure on the bottom left).

\paragraph{Density control and position learning}\label{app:exp-ablation-density-control}
In our experiments, we found that resetting opacity for all of the Gaussians during densification can help to reduce floaters. As shown in Figure~\ref{fig:ablation_opacity_reset}, with opacity reset, there are much less floaters in the final result (bottom) compared with the case without opacity reset (top). Note that, in our experiment, we choose to reset the opacity to 0.005 at the iteration of 1000 based on grid search. 

Besides opacity reset, we also found the representation of 3D Gaussians is very sensitive to the learning rate of the positions (\ie, xyz coordinates). As shown in Figure~\ref{fig:ablation_positional_lr}, with a slightly large learning rate (0.0064), the geometry gets diverged due to the diversification process. This is aligned with the result from original 3D Gaussians paper~\cite{kerbl3Dgaussians}. Even under their reconstruction task, which has more regularization (\ie image supervision) comparing with our generation task, the original method still uses a really small position learning rate as 0.00064, which essentially does not allow the centroids of the 3D Gaussians moving much. Instead, the fine geometry is forced to be learned by density control (densification and pruning). 

\begin{figure*}[ht]
  \includegraphics[width=\textwidth]{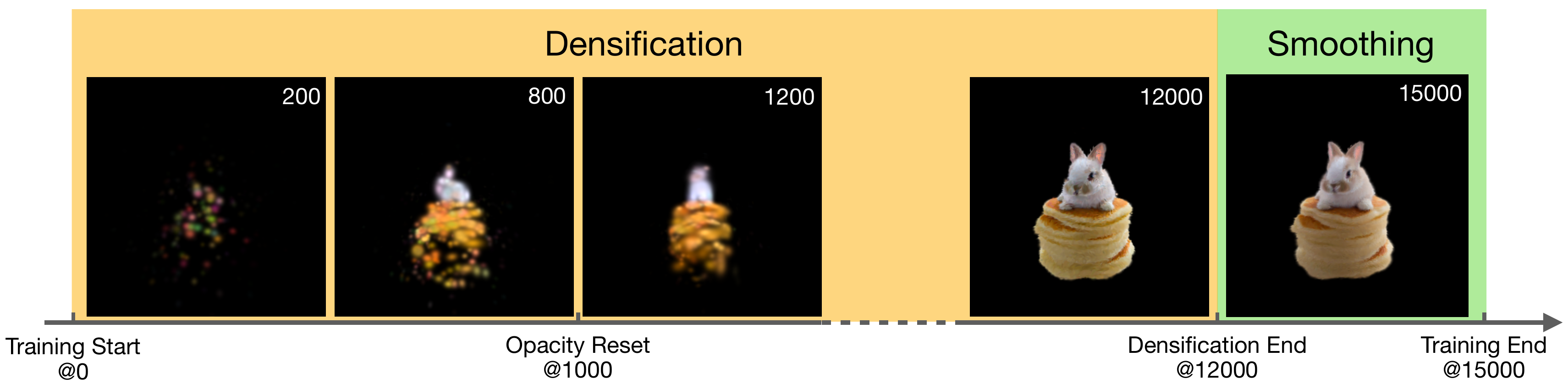}
  \caption{Density control schedule. We randomly initialize points and apply density control (densification and pruning) to obtain the coarse geometry and texture. Then an additional smoothing step is followed in order to remove the spiky artifacts as introduced by densification.}
  \label{fig:stage-1-schedule}
  \vskip -0.2in
\end{figure*}


\section{Ablation on Two-Stage Training} \label{app:ablation-two-stage-training}

\paragraph{Benefit from the coarse-to-fine training paradigm.} 
\cref{fig:ablation_gs_stage_1} shows the first stage result (\ie, training from scratch) using Stable Diffusion model (left) and DeepFloyd model (right) for both of the geometry primitives \GS and iNGP. 
Although there is a sharper texture from the high-resolution Stable Diffusion model, the overall geometry is worse than the result from the coarse DeepFloyd model.
As shown in~\cref{fig:ablation_gs_stage_2}, after finetuning with the diffusion models trained with high resolution images (Stable Diffusion or DeepFloyd with super-resolution module), we can get a 3D model with much higher fidelity, while also keeps the good geometry that is learned from the first stage.

\begin{figure}[!hp]
  \vskip -0.05in
  \includegraphics[width=\linewidth]{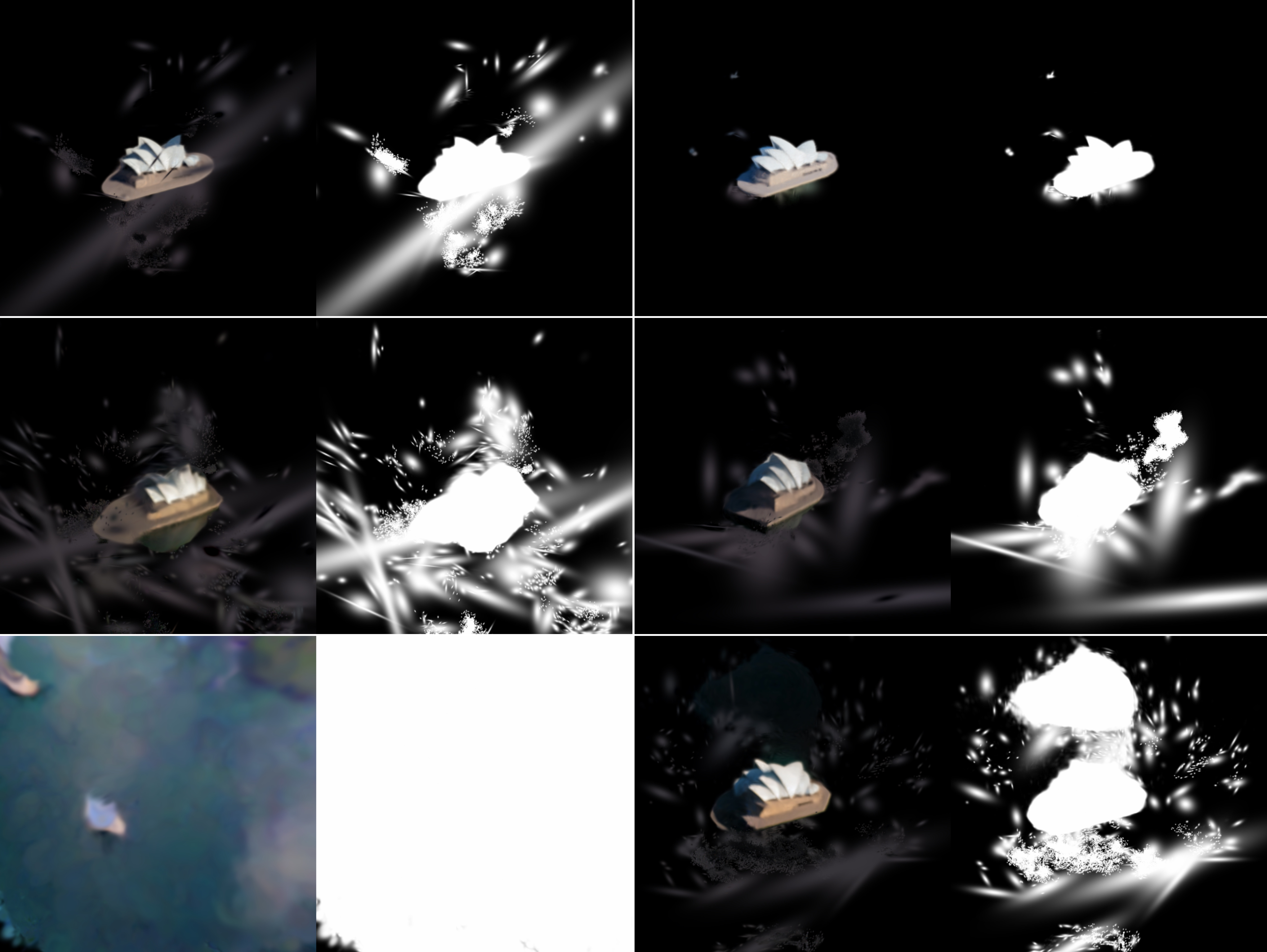}
  \caption{Ablation study for \GS initialization schemes with prompt: \textit{a zoomed out DSLR photo of the Sydney opera house, aerial view}. Left Column: Fix initial opacity levels. Right Column: Opacity initialization based on distance to center of scene. Top Row: 1K starting points. Middle Row: 10K starting points. Bottom Row: 100K starting points. }
  \label{fig:init}
  \vskip -0.25in
\end{figure}
\paragraph{Benefit of the use of latent-space diffusion model in the second stage learning.} 
As shown in~\cref{fig:ablation_gs_stage_2}, when finetuning from the first stage model trained with the coarse DeepFloyd model, both Stable Diffusion and DeepFloyd with super-resolution module can achieve better geometry and texture, as they are trained with high resolution images. 
However, if we compare the resulting images, (\eg, the texture of basket and the fine hairs from bunny) the DeepFloyd result is lacking details, while the Stable Diffusion model can produce both better texture and sharper geometry.
As mentioned earlier, this is due to the image-based guidance (\ie, DeepFloyd) has more adverse effect to the view consistency of the 3D model, while the guidance from the latent-space diffusion model (\ie, Stable Diffusion) is less sensitive due to the feature compression from its image encoder. 

\begin{figure}[!hp]
  \includegraphics[width=\linewidth]{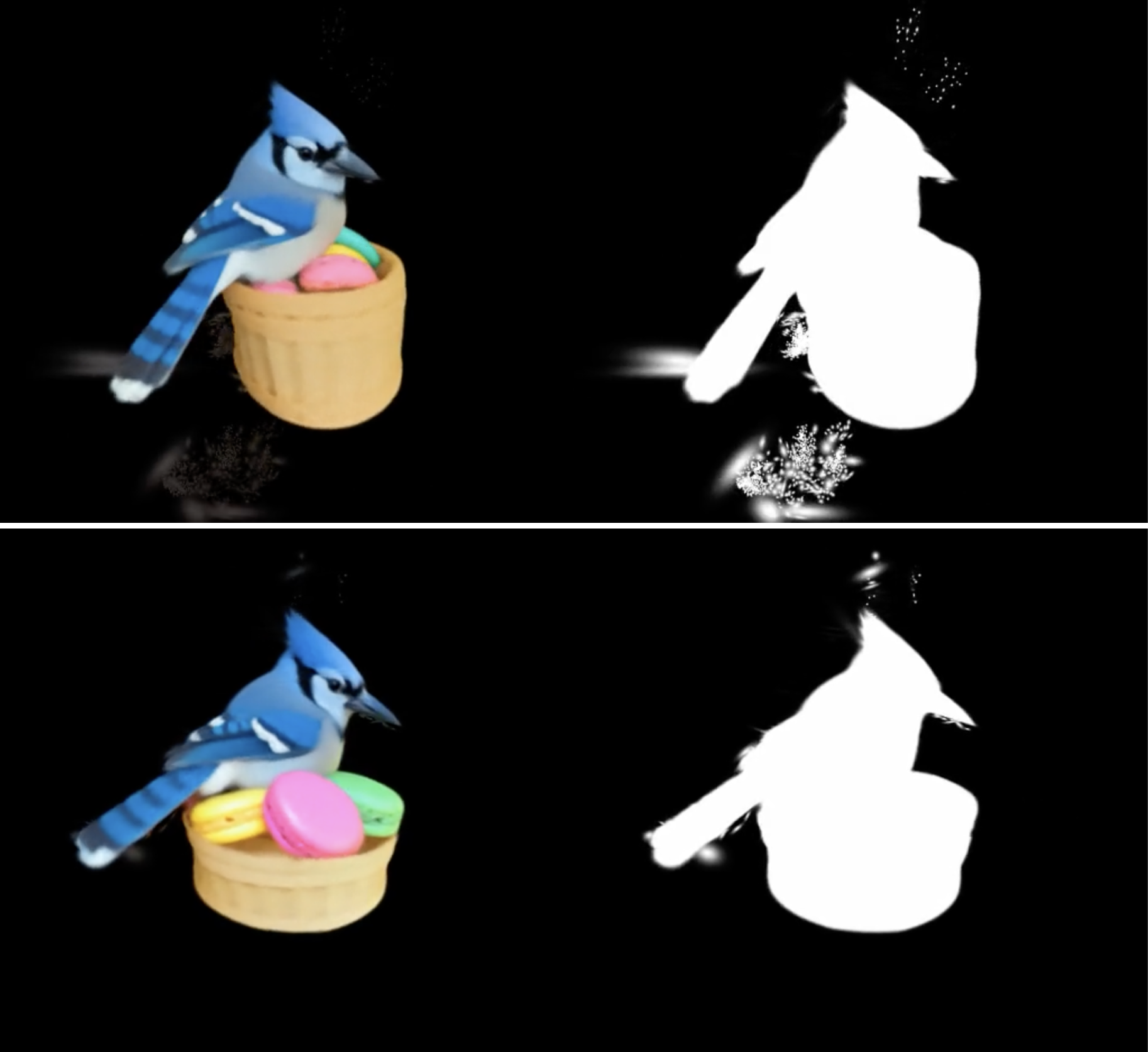}
  \caption{Resetting opacity during densification can help reduce floaters as shown in the opacity renderings on the right. Top: without opacity reset; bottom: with opacity reset.}
  \label{fig:ablation_opacity_reset}
  \vskip -0.2in
\end{figure}

\begin{figure}[!hp]
  \includegraphics[width=\linewidth]{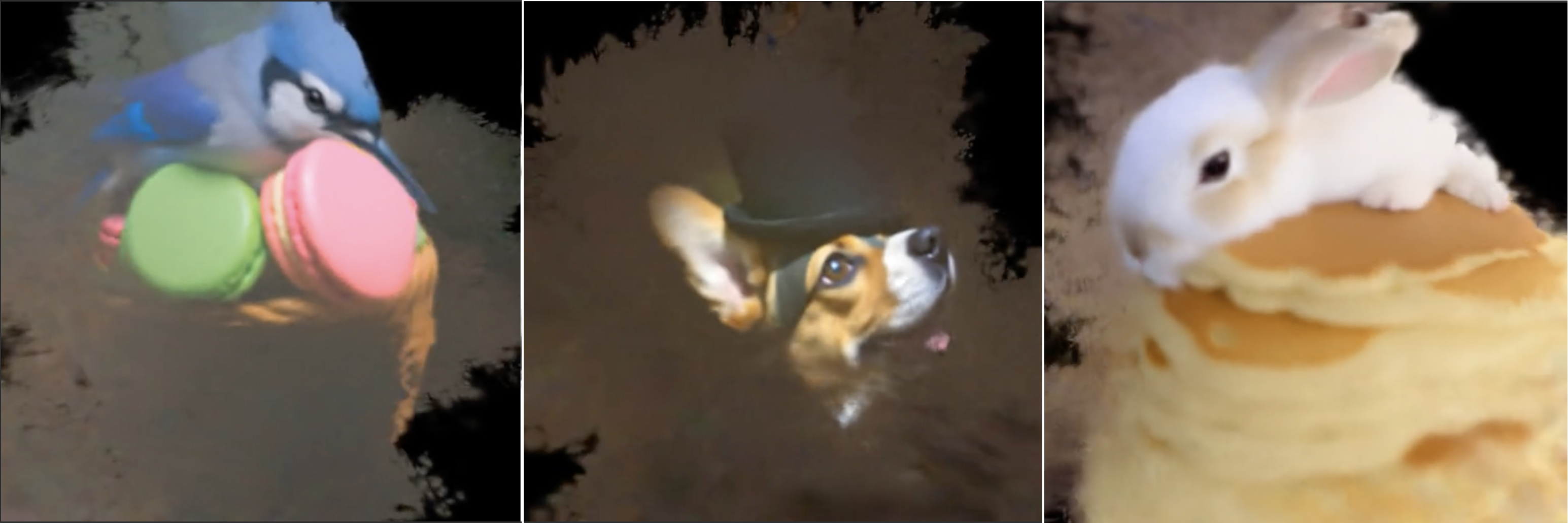}
  \caption{Using an inappropriate learning rate for position updates can readily lead to geometric divergence.}
  \label{fig:ablation_positional_lr}
  \vskip -0.5in
\end{figure}

\begin{figure*}[htbp]
  \includegraphics[width=\linewidth]{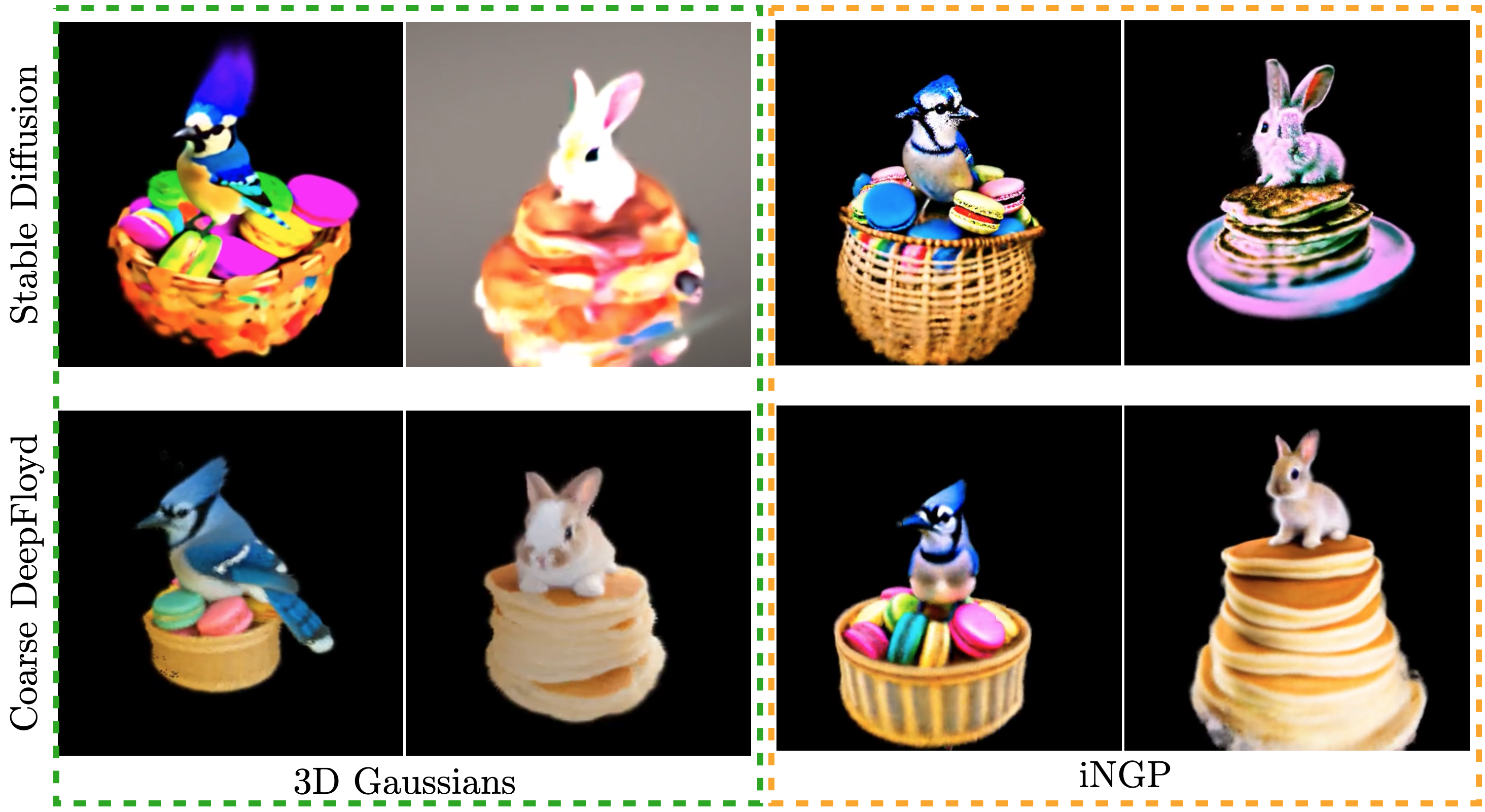}
  \caption{Result from different diffusion models when training from scratch.}
  \label{fig:ablation_gs_stage_1}
\end{figure*}

\begin{figure*}[htbp]
  \includegraphics[width=\linewidth]{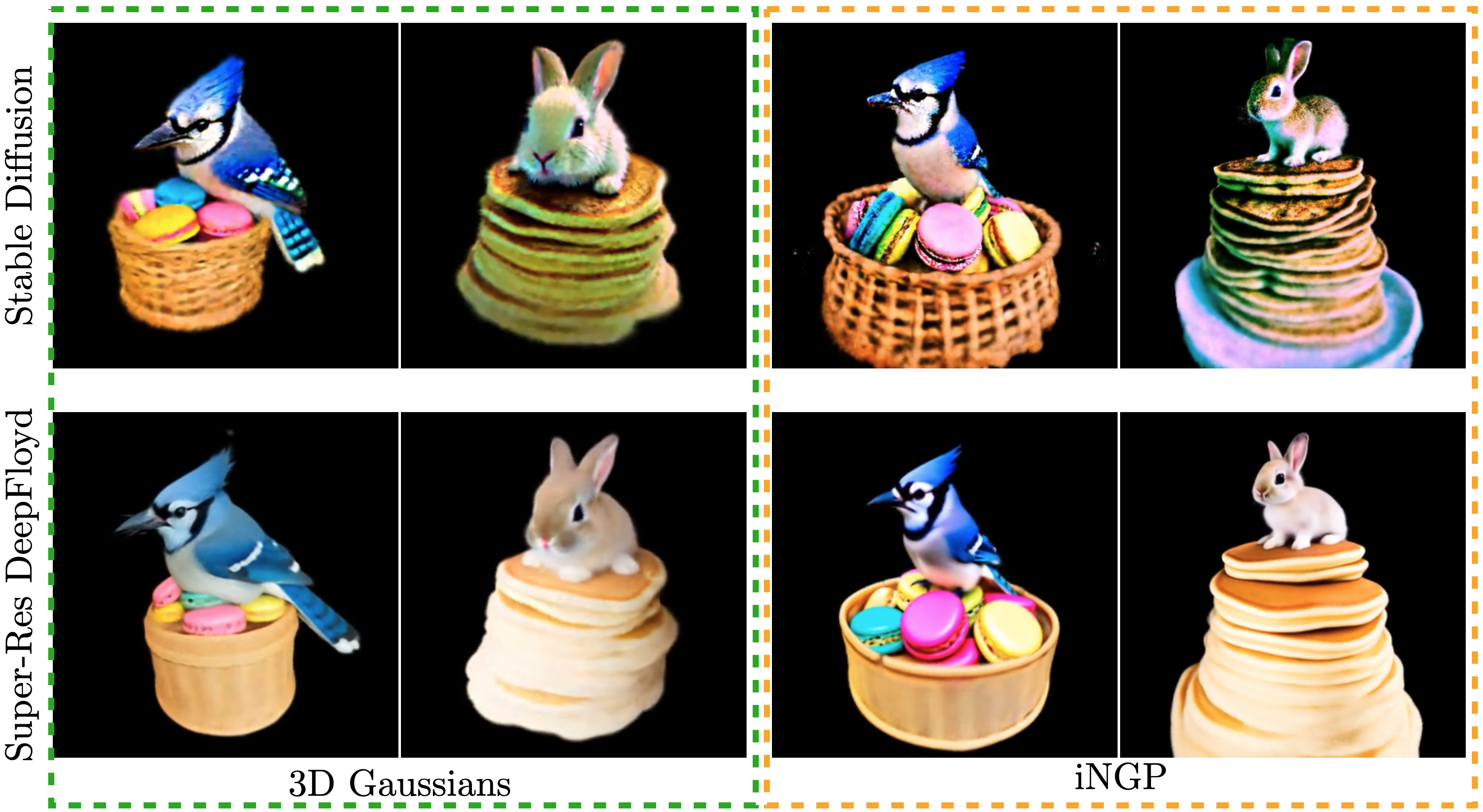}
  \caption{Result from different diffusion models when finetuning from the first stage model.}
  \label{fig:ablation_gs_stage_2}
\end{figure*}

\end{document}